\documentclass{article}
\usepackage[english]{babel} % English language/hyphenation.
\usepackage[T1]{fontenc} % Use 8-bit output encoding.
\usepackage[utf8]{inputenc} % Can use UTF-8 in the source files.
\usepackage[babel]{microtype} % Improves appearance of text.

\usepackage{url}
\usepackage{csquotes}
\usepackage{amsmath,amsthm,amssymb,bm,mathtools}
\usepackage[pdftex]{graphicx}
\graphicspath{{fig/}}
\usepackage{enumitem}
\usepackage{booktabs}
\usepackage{siunitx}
\usepackage{makecell}
\usepackage{subfig}
\usepackage{balance}

\usepackage[accepted]{icml2020}

\newlist{enuminline}{enumerate*}{1}
\setlist[enuminline,1]{label=(\itshape\alph*\upshape)}

\newtheorem{theorem}{Theorem}
\newtheorem{lemma}{Lemma}

\newcommand{\tr}{T}

\usepackage{amsmath,amsfonts,bm}

\def\eqref#1{equation~\ref{#1}}
\def\1{\bm{1}}

\def\vmu{{\bm{\mu}}}
\def\vnu{{\bm{\nu}}}

\def\vw{{\bm{w}}}
\def\vx{{\bm{x}}}

\def\vz{{\bm{z}}}

\def\mI{{\bm{I}}}

\def\mX{{\bm{X}}}

\def\mSigma{{\bm{\Sigma}}}
\def\mPsi{{\bm{\Psi}}}

\DeclareMathAlphabet{\mathsfit}{\encodingdefault}{\sfdefault}{m}{sl}
\SetMathAlphabet{\mathsfit}{bold}{\encodingdefault}{\sfdefault}{bx}{n}

\newcommand{\E}{\mathbb{E}}

\newcommand{\R}{\mathbb{R}}

\newcommand{\KL}{D_{\mathrm{KL}}}

\DeclareMathOperator*{\argmax}{arg\,max}
\DeclareMathOperator*{\argmin}{arg\,min}

\DeclareMathOperator{\Tr}{Tr}

\begin{document}

\twocolumn[
\icmltitle{Scalable and Efficient Comparison-based Search without Features}

\begin{icmlauthorlist}
% Daniyar Chumbalov \And Lucas Maystre \And Matthias Grossglauser
\icmlauthor{Daniyar Chumbalov}{EPFL}
\icmlauthor{Lucas Maystre}{Spotify}
\icmlauthor{Matthias Grossglauser}{EPFL}
\end{icmlauthorlist}

\icmlaffiliation{EPFL}{School of Computer and Communication Sciences, EPFL,
Lausanne, Switzerland.}
\icmlaffiliation{Spotify}{Spotify}

\icmlcorrespondingauthor{Daniyar Chumbalov}{daniyar.chumbalov@epfl.ch}

% \icmlkeywords{Machine Learning, ICML}

\vskip 0.3in
]

% \maketitle

\printAffiliationsAndNotice{}

\begin{abstract}
We consider the problem of finding a target object $t$ using pairwise comparisons,
by asking an oracle questions of the form \emph{``Which object from the pair $(i,j)$ is more similar to $t$?''}.
Objects live in a space of latent features, from which the oracle generates noisy answers.
First, we consider the {\em non-blind} setting where these features are accessible.
We propose a new Bayesian comparison-based search algorithm with noisy answers; it has low computational complexity yet is efficient in the number of queries.
We provide theoretical guarantees, deriving the form of the optimal query and proving almost sure convergence to the target $t$.
Second, we consider the \emph{blind} setting, where the object features are hidden from the search algorithm.
In this setting, we combine our search method and a new distributional triplet embedding algorithm into one scalable learning framework called \textsc{Learn2Search}.
We show that the query complexity of our approach on two real-world datasets is on par with the non-blind setting, which is not achievable using any of the current state-of-the-art embedding methods. Finally, we demonstrate the efficacy of our framework by conducting an experiment with users searching for movie actors.
\end{abstract}

\section{Introduction}
\label{sec:intro}

% define the significance of the search/IR problem; approaches: keyword or other explicit query
Finding a target object among a large collection of $n$ objects is the central problem addressed in information retrieval (IR).
For example, in web search, one finds relevant web pages through a query expressed as one or several keywords.
Of course, the form of the query depends on the object type; other forms of search queries in the literature include finding images similar to a query image \cite{datta2008image}, and finding subgraphs of a large network isomorphic to a query graph \cite{sun2012efficient}.
%
% limitation: need to be able to formulate explicit query, in terms of features of object; for some types of db, this is not natural; however, humans can compare
A common feature of this classic formulation of the search problem is the need to express a meaningful query.
However, this is often a non-trivial task for a human user.
For example, most people would struggle to draw the face of a friend accurately enough that it could be used as a query image.
But they would be able to confirm with near-total certainty whether a photograph is of their friend---we have no real doppelgängers in the world.
In other words, comparing is cognitively much easier than formulating a query.
This is true for many complex or unstructured types of data, such as music, abstract concepts, images, videos, flavors, etc.

% focus on comparison search: finding target by repeated comparisons to query objects; but in addition, assume we know nothing about feature space; only information on organization of space: past searches
In this work, we develop models and algorithms for searching a large database by comparing a sequence of sample objects relative to the target, thereby obviating the need for any explicit query.
The central problem concerns the choice of the sequence of samples (e.g., pairs of faces that the user compares with a target face she remembers).
These query objects have to be chosen in such a way that we learn as much as possible about the target, i.e., shrink the set of potential targets as quickly as possible.
This is closely related to a classic problem in active learning \cite{settles2012active, mackay1992information}, assuming we have a meaningful set of features available for each object in the database.

% assumption: objects live in lower-dimensional space than raw data; example: pictures of faces: can be usefully described by tends of features, arguably lives in low-dimensional space; goal of our approach: use past searches to embed in this low-dim space, in order to improve future searches
It is natural to assume that the universe of objects lives in some low-dimensional latent feature space, even though the raw objects might be high-dimensional (images, videos, sequences of musical notes, etc.).
For example, a human face, for the purposes of a similarity comparison, could be quite accurately described by a few tens of features, capturing head shape, eye color, fullness of lips, gender, age, and the like \cite{chang2017code}.

%We assume in this paper that this latent space is difficult to extract from the raw data, e.g., through face recognition techniques.
%Even though it is often the case that some noisy features could be extracted from the raw data, we choose to focus on a completely {\em blind} setting.
The key component of our framework is a noisy oracle model that describes how a user responds to comparison queries relative to a target, based on the objects' embedding in the latent feature space.
The noise model implicitly determines which queries are likely to extract useful information from the user;
it is similar to pairwise comparison models such as those of \citet{thurstone1927law} or \citet{bradley1952rank}, but here the comparison is always relative to a target.
The compared samples and the target are all represented as points in the feature space.
Our contributions in this context are threefold.
First, we develop an efficient search algorithm, which
\begin{enuminline}
\item generates query pairs and
\item navigates towards the target.
\end{enuminline}
This algorithm is provably guaranteed to converge to the target, due to the specific form of our comparison model and key properties of the search algorithm.
In contrast with prior adaptive search methods, where generating a query pair is typically expensive, our algorithm is also computationally efficient.

% a form of reinforcement learning: min avg # of queries to find target over time
We also consider the more difficult, but realistic, scenario where the objects still live in a low-dimensional feature space, but the search algorithm does not have access to their embedding.
In this {\em blind} setting, the latent features of an object manifest themselves only through the comparison queries.
Although this assumption clearly makes the problem challenging, it also has the advantage of yielding a completely generic approach.
We develop a method to generate a latent embedding of the objects from past searches, in order to make future searches more efficient.
The goal of the method is to minimize the number of queries to locate a target; this is achieved by gradually improving the latent embedding of all the objects based on past searches.
Unlike with previous triplet embedding approaches that use low-dimensional point estimates to represent objects, we embed each object as a Gaussian distribution, which captures the uncertainty in the exact position of the object given the data.
We show that our approach leads to significant gains in a real-world application, where users search for movie actors by comparing faces:
Our framework learns useful low-dimensional representations based on comparison outcomes that lead to decreasing search cost as more data is collected.

% contributions: scalability: we are focused on large # of objects $n$; two components: model for noisy answers, given by embedding; given embedding, efficient (scalable) algo to generate queries; given past queries&answers, algo to embed objects in low-dim space; put together, first combined comp-based search & embedding, with noise, scalable to large $n$
%We combine the search and the embedding problem into a single probabilistic framework.
%(TODO) commenting this out for now, might incorporate back later
%Second, we develop an embedding algorithm that computes a latent embedding of all the objects based on past searches.
%Third, we show experimentally that blind comparison-based search is tractable:
%starting from a random embedding, we can drastically reduce the number of queries needed to identify a target after a few hundred searches.
%Our approach is, to the best of our knowledge, the first complete and scalable approach to blind comparison-based search.

% forward pointers
The rest of the paper is structured as follows.
In Section~\ref{sec:relwork}, we give an overview of related work.
In Section~\ref{sec:preliminaries}, we define the problem of comparison-based search and define our user (oracle) model.
%Then in Section~\ref{sec:search}, we propose \textsc{GaussSearch}: our new scalable method for finding the target object through pairwise comparisons if the features of the objects are known.
In Section~\ref{sec:search}, we describe our interactive search method in the non-blind setting, establish its favorable scalability, and prove its convergence properties.
In Section~\ref{sec:embedding}, we focus on the blind setting, and introduce our distributional embedding method \textsc{GaussEmbed} and its integration into the combined \textsc{Learn2Search} framework.
We demonstrate the performance of our algorithms on synthetic and real world datasets in Section~\ref{sec:experiments} in both non-blind and blind settings, and also present the results of an experiment invloving human oracles searching for movie.
Finally, we give our concluding remarks in Section~\ref{sec:conclusion}.
For conciseness, we defer the full proofs of our theorems to the supplementary material.

\section{Related Work}
\label{sec:relwork}

%TODO: add something else? other literature? write about weaknesses of the papers, and what we bring new? move the section in the place before experiments?

\paragraph{Comparison-Based Search.}
For the general \emph{active learning} problem, where the goal is to identify a target hypothesis by making queries, the Generalized Binary Search method~\cite{dasgupta2005analysis, nowak2008generalized} is known to be near-optimal in the number of queries if the observed outcomes are noiseless.
In the noisy case, \citet{nowak2009noisy} and \citet{golovin2010near} propose objectives that are greedily minimized over all possible queries and outcomes in a Bayesian framework with a full posterior distribution over the hypothesis. We adapt the method suggested by \citet{golovin2010near} to our setting and investigate its performance in comparison to the one proposed in this paper later in Section~\ref{sec:experiments}.

Adaptive search through relevance feedback by using preprocessed features was studied in the context of image retrieval systems by \citet{cox2000bayesian, fang2005experiments, ferecatu2009statistical, suditu2012iterative} in a Bayesian framework and by using different user answer models.

% TODO: the following para is about blind -> should it be moved to triplet embedding para below?
When the features are hidden, finding a target in a set of objects by making similarity queries is explored by \citet{tschopp2011randomized} and \citet{karbasi2012comparison} in the noiseless case. To deal with erroneous answers, \citet{kazemi2018comparison} consider an augmented user model that allows a third outcome to the query, ``?''; it can be used by a user to indicate that they do not know the answer confidently. With this model, they propose several new search techniques under different assumptions on the knowledge of relative object distances. \citet{karbasi2012comparison} also briefly discussed the case of a noisy oracle with a constant mistake probability and proposed to treat it with repeated queries. \citet{haghiri2017comparison} consider the problem of nearest-neighbor search using comparisons in the noiseless setting.

\citet{brochu2008active} study the comparison-based search problem in continuous space.
Their approach relies on maximizing a surrogate GP model of the valuation function.
In contrast to their work, our approach is restricted to a specific valuation function (based on the distance to the target), but our search algorithm is backed by theoretical guarantees, its running time is independent of the number of queries and scales better with the number of objects.
\citet{chu2005extensions} and \citet{houlsby2011bayesian} study a similar problem where the entire valuation function needs to be estimated.
Our approach is inspired from this line of work and also uses the expected information gain to drive the search algorithm.
However, our exact setup enables finding a near-optimal query pair simply by evaluating a closed-form expression, as opposed to searching exhaustively over all pairs of objects.
Comparison-based search under noise is also explored in \citet{canal2019icml}; the authors independently propose an oracle model similar to ours.

%The main drawbacks of these methods was their poor scalability of the next query finding procedure as $n$ grows large.

\paragraph{Triplet Embedding.}
The problem of learning objects embedding based on triplet comparisons has been studied extensively in recent years.
\citet{jamieson2011low} provides a lower bound on the number of triplets needed to uniquely determine an embedding that satisfies all $O(n^3)$ relational constraints and describe an adaptive scheme that builds an embedding by sequentially selecting triplets of objects to compare. 
% This algorithm cannot be applied in our problem as we collect triplets from the completed searches and cannot choose them adaptively.

More general algorithms for constructing ordinal embeddings from a \emph{given} set of similarity triplets, under different triplet probability models, are proposed by \citet{agarwal2007generalized}, \citet{tamuz2011adaptively} and \citet{van2012stochastic}.
\citet{amid2015multiview} modify the optimization problem of \citet{van2012stochastic} to learn at once multiple low-dimensional embeddings, each corresponding to a single latent feature of an object.
Theoretical properties for a general maximum likelihood estimator with an assumed probabilistic generative data model are studied by \citet{jain2016finite}. 

An alternative approach to learning ordinal data embedding is suggested in \citet{kleindessner2017kernel}, where the authors explicitly construct kernel functions to measure similarities between pairs of objects based on the set of triplet comparisons.
Finally, \citet{heim2015efficient} adapt the kernel version of \citet{van2012stochastic} for an online setting, when the set of observable triplets is expanding over time.
The authors use stochastic gradient descent to learn the kernel matrix and to exploit sparsity properties of the gradient.
Although this work is closely related to our scenario, the kernel decomposition, which is $O(n^3)$ in time, would be too computationally expensive to perform in the regimes of interest to us.
%the underlying technique would be too expensive to use in our scheme, since after each kernel update we would have to perform kernel decomposition, in order to obtain embeddings of the objects for the search part, which would lead to $\mathrm{O}(n^3)$ complexity, comparable to the total number of all possible triplets. 

Finally, \citet{ghosh2019landmark} and \citet{anderton2019scaling} both propose landmark-based triplet embedding methods that scale well for large values of $n$ but assume direct access to the oracle in order to ask to compare adaptively selected triplets.

\section{Preliminaries}
\label{sec:preliminaries}

Let us consider $n$ objects denoted by integers $[n] = \{1, 2, \dots, n \}$.
Our goal is to build an interactive comparison-based search engine that would assist users in finding their desired \emph{target} object $t \in [n]$ by sequentially asking questions of the form ``among the given pair of objects $(i,j)$, which is closer to your target?'' and observing their answers $y \in \{i,j\}$.
Formally, at each step of this search, we collect an answer $y$ to the query $(i,j)$.
We then use this information to decide on the next pair of objects $(i', j')$ to show to the user, until one of the elements in the query is recognized as the desired target.
% We assume that except for the comparison data collected during past searches, we do not have access to any other type of representational information about the objects (such as images, text descriptions, etc.).

We assume that the objects have associated latent feature vectors $\mathcal{X} = \{ \vx_1, \vx_2, \dots, \vx_n \}, \vx_i \in \R^d$, that reflect their individual properties and that are intuitively used by users when answering the queries.
In fact, each user response $y$ can be viewed as an outcome of a perceptual judgement on the similarity between objects that is based on these representations.
A natural choice of quantifying similarities between objects is the Euclidean distance between hidden vectors in $\R^d$:
\begin{align*}
  y(i,j ~|~ t) = i ~~\iff~~ ||\vx_i - \vx_t|| < ||\vx_j - \vx_t||.
\end{align*}
Different users might have slightly different perceptions of similarity. Even for a single user, it might be difficult to answer every comparison confidently, and as such we might observe inconsistencies across their replies. We postulate that this happens when two objects $\vx_i$ and $\vx_j$ are roughly equidistant from the target $\vx_t$, and less likely when the distances are quite different. In other words, answers are most noisy when the target is close to the decision boundary, i.e., the bisecting normal hyperplane to the segment between $\vx_i$ and $\vx_j$, and less noisy when it is far away from it.

We consider the following probabilistic model that captures the possible uncertainty in users' answers (\emph{probit model}):
\begin{align}
  p(y(i,j~|~t) = i) &= p(\vx_t^T \vw_{ij} + b_{ij} + \varepsilon > 0) \nonumber \\
  &= \Phi\left(\frac{\vx_t^T \vw_{ij} + b_{ij}}{\sigma_\varepsilon}\right), \label{eq:probit}
\end{align}
where $\vx_i, \vx_j, \vx_t \in \R^d$ are the coordinates of the query points and the target, respectively,
\begin{align*}
h_{ij} = (\vw_{ij}, b_{ij}) = \left( \frac{\vx_i - \vx_j}{\lVert \vx_i - \vx_j \rVert}, \frac{\Vert \vx_j \rVert^2 - \lVert \vx_i \rVert^2}{2\lVert \vx_i - \vx_j \rVert} \right)
\end{align*}
is the bisecting normal hyperplane to the segment between $\vx_i$ and $\vx_j$, $\varepsilon \sim \mathcal{N}(0, \sigma_\varepsilon^2)$ is additive Gaussian noise, and $\Phi$ is the Standard Normal CDF.
Indeed, if $\vx_t$ is on the hyperplane $h_{ij}$, the answers to queries are pure coin flips, as the target point is equally far from both query points. Everywhere else answers are biased toward the correct answer, and the probability of the correct answer depends only on the distance of $\vx_t$ to the hyperplane $h_{ij}$.

% When $\vx_t$ is close to, but not necessarily on the hyperplane, the effect of random noise still plays a major role in the user answer decisions.
% Finally, when $\vx_t$ starts to get farther from the decision bound, probability of the correct outcome increases.

This model is reminiscent of pairwise comparison models such as those of \citet{thurstone1927law} or \citet{bradley1952rank}, e.g., where $p(y(i,j~|~t) = i) = \Phi[s(\vx_i) - s(\vx_j)]$ and $s(\vx) = -\lVert \vx - \vx_t \rVert_2^2$ \cite{agarwal2007generalized, jain2016finite}.
These models have the undesirable property of favoring distant query points: given any $\vx_i, \vx_j \ne \bm{0}$, it is easy to verify that the pair $(2\vx_i, 2\vx_j)$ is strictly more discriminative for any target that does not lie on the bisecting hyperplane.
Our comparison model is different: in (\ref{eq:probit}), the outcome probability depends on $\vx_i$ and $\vx_j$ \emph{only} through their bisecting hyperplane.

\renewcommand{\algorithmicrequire}{\textbf{Input:}}
\renewcommand{\algorithmicensure}{\textbf{Output:}}

\section{Search with Known Features}
\label{sec:search}

In this section, we describe our algorithm for interactive search in a set of objects $[n]$ when we have access to the features $\mathcal{X}$.  We are interested in methods that are
\begin{enuminline}
\item \emph{efficient} in the average number of queries per search, and
\item \emph{scalable} in the number of objects.
\end{enuminline}
Scalability requires \emph{low computational complexity} of the procedure for choosing the next query. As we expect users to make mistakes in their answers, we also require our algorithms to be \emph{robust} to noise in the human feedback.

\paragraph{Gaussian Model.}
Due to the sequential and probabilistic nature of the problem, we take a Bayesian approach in order to model the uncertainty about the location of the target.
In particular, we maintain a $d$-dimensional Gaussian distribution $\mathcal{N} ( \hat{\vx} ; \vmu , \mSigma)$ that reflects our current belief about the position of the target point $\vx_t$ in $\R^d$.
We model user answers with the probit likelihood (\ref{eq:probit}), and we apply approximate inference techniques for updating $\vmu$ and $\mSigma$ every time we observe a query outcome.
The space requirement of the model is $O(d^2)$.

The motivation behind such a choice of parametrization is that
\begin{enuminline}
\item the size of the model does not depend on $n$, guaranteeing scalability,
\item \label{itm:eig} one can characterize a general pair of points in $\R^d$ that maximizes the expected information gain, and
\item the sampling scheme that chooses the next pair of query points informed by~\ref{itm:eig} is simple and works extremely well in practice.
\end{enuminline}

\subsection{Choosing the Next Query}

To generate the next query, we follow a classical approach from information-theoretic active learning \cite{mackay1992information}: find the query that minimizes the expected posterior uncertainty of our estimate $\hat{\vx}$, as given by its entropy.
More specifically, we find a pair of points $(\vx_i, \vx_j)$ that maximizes the expected \emph{information gain} for our current estimation of $\vx_t$ at step $m$:
\begin{align*}
  I[\hat{\vx}; y \mid (\vx_i, \vx_j)] = H(\hat{\vx}) - \E_{ y \mid \hat{\vx}, \vx_i, \vx_j} [ H(\hat{\vx} \mid y) ],
\end{align*}
where $\hat{\vx} \sim \mathcal{N} (\vmu , \mSigma)$ is the current belief about the target position and  $y \sim p(y \mid \hat{\vx}, \vx_i, \vx_j)$ is the current belief over the answer to the comparison between $\vx_i$ and $\vx_j$.

Exhaustively evaluating the expected information gain over all $O(n^2)$ pairs for each query (e.g., as done in \citealp{chu2005extensions}) is prohibitively expensive.
Instead, we propose a more efficient approach.
 % that runs in time $O(\log n)$ with $O(n \log{n})$ preprocessing using a $kd$-tree, that is done only once for an embedding.
Recall from equation~(\ref{eq:probit}) that the comparison outcome probability $p(y \mid \hat{\vx}, \vx_i, \vx_j)$ depends on $(\vx_i, \vx_j)$ only through the corresponding bisecting normal hyperplane $h_{ij}$.
%, if the position of the target $\vx_t$ and parameters of a hyperplane $(\vw, b)$ are fixed, any pair of points in $\R^d$ (not necessarily represented in $\mX$) that generates this hyperplane would have the same probability distribution on the query outcomes.
Therefore, instead of looking for the optimal pair of points from $\mathcal{X}$, we first find the hyperplane $h$ that maximizes the expected information gain.
Following \citet{houlsby2011bayesian}, it can be rewritten as:
\begin{align}
  I[\hat{\vx}; y \mid h]
    &= H(\hat{\vx}) - \E_{ y \mid \hat{\vx}, h} [ H(\hat{\vx} \mid y) ] \nonumber \\
    &= H(y \mid \hat{\vx}, h) - \E_{\hat{\vx}} [ H(y \mid \hat{\vx}, h) ]. \label{eq:utility}
\end{align}
The following theorem gives us the key insight about the general form of a hyperplane $h = (\vw, b)$ that optimizes this utility function by splitting the Gaussian distribution into two equal ``halves'':

\begin{theorem}
\label{thm:optimal}
Let $\hat{\vx} \sim \mathcal{N}(\vmu, \mSigma)$, and let $\mathcal{H} = \{ (\vw, b) \mid \lVert \vw \rVert = 1, \vw^\tr \vmu + b = 0 \}$ be the set of all hyperplanes passing through $\vmu$.
Then,
\begin{align*}
%(\vw^\star, b^\star) \in \argmax_{(\vw, b) \in \mathcal{H}} U((\vw, b), \vx) \\
%\iff \vw^\star \in \argmax_{\lVert \vw \rVert = 1} \vw^\tr \mSigma \vw.
\argmax_{h \in \mathcal{H}} I[\hat{\vx}; y \mid h]
= \argmax_{(\vw, b) \in \mathcal{H}} \vw^\tr \mSigma \vw.
\end{align*}
\end{theorem}
\begin{proof}[Proof sketch]
Let $h^\star = (\vw^\star, b^\star)$ be an optimal hyperplane.
Since $\vw^\tr \vmu + b = 0$, we have $p(y \mid \hat{\vx}, h) \equiv 1/2$ and $H(y \mid \hat{\vx}, h) \equiv 1$.
Thus, $h^\star$ is such that $\E_{\hat{\vx}} [ H(y \mid \hat{\vx}, h) ]$ is minimal.
We show that
\begin{align*}
\E_{\hat{\vx}} [ H(y \mid \hat{\vx}, h) ] = \E_{z} [ H( \Phi(z) ) ],
\end{align*}
where $z \sim \mathcal{N}(0, \vw^\tr \mSigma \vw)$.
As the binary entropy function has a single maximum in $1/2$, the expectation on the r.h.s. is minimized when the variance of $z$ is maximized.
\end{proof}
In other words, Theorem~\ref{thm:optimal} states that any optimal hyperplane is orthogonal to a direction that maximizes the variance of $\hat{\vx}$.

Consequently, we find a hyperplane $h$ that maximizes the expected information gain via an eigenvalue decomposition of $\mSigma$.
However, we still need to find a pair of objects $(i,j)$ to form the next query.
We propose a sampling strategy based on Theorem~\ref{thm:optimal}, which we describe in Algorithm~\ref{alg:sample}.
%We maintain a set of points $\mathcal{U}$ that were used as a query.
After computing the optimal hyperplane $h^\star$, we sample a point $\vz_1$ from the current Gaussian belief over the location of the target, and reflect it across $h^\star$ to give the second point $\vz_2$.
Then our query is the closest pair of (distinct) unused points $(\vx_i, \vx_j)$ from the dataset to $(\vz_1,\vz_2)$.

Assuming that the feature vectors $\mathcal{X}$ are organized into a $k$-d tree \cite{bentley1975multidimensional} at a one-time computational cost of $O(n \log{n})$ for a given $\mathcal{X}$, Algorithm~\ref{alg:sample} runs in time $O(\log{n}+d^3)$.
This includes $O(d^3)$ for the eigenvalue decomposition, $O(d^2)$ for sampling from a $d$-dimensional Gaussian distribution, $O(d)$ for mirroring the sample, and $O(\log{n})$ on average for finding the closest pair of points from the dataset using the $k$-d tree.

If needed, computing the eigenvector that corresponds to the largest eigenvalue can be well approximated via the power method with $O(d^2)$ complexity (assuming a constant number of iterations). This will bring down the total complexity of Algorithm~\ref{alg:sample} to $O(\log{n}+d^2)$.

\begin{algorithm}[ht]
  % \centering
  \caption{\textsc{SampleMirror}}\label{alg:sample}
  \begin{algorithmic}[1]
    \REQUIRE current belief $\mathcal{N}(\vmu, \mSigma)$, $\mathcal{X}$, $\mathcal{U}$
    \ENSURE query pair $(\vx_i, \vx_j)$
    \STATE Compute optimal hyperplane $h^\star$ for $(\vmu , \mSigma)$.
    \STATE Sample a point $\vz_1$ from $\mathcal{N} ( \vmu , \mSigma)$.
    \STATE Obtain $\vz_2$, as the reflection of $\vz_1$ across $h^\star$.
    \STATE Find objects $i$ and $j$, s.t. $i \neq j$ and $i,j \not\in \mathcal{U}$, with the closest representations $\vx_i$ and $\vx_j$ to $\vz_1$ and $\vz_2$, respectively.
  \end{algorithmic}
\end{algorithm}

Note that the actual hyperplane defined by the obtained pair $(\vx_i, \vx_j)$ in Algorithm~\ref{alg:sample} does not, in general, coincide with the optimal one.
Nevertheless, the hyperplane $h_{ij}$ approximates $h^\star$ increasingly better as $n$ grows.

% for large enough $n$ we could make a reasonable assumption of high density of points $\{\vx_i\}_{i=1}^n$ in $\R^d$, and hance expect $\vh_{i,j}$ to be a good enough approximation for $\hat{\vh}$.

\subsection{Updating the Model}

Finally, after querying the pair $(i, j)$ and observing outcome $\bar{y}$, we update our belief on the location of target $\hat{\vx}$.
Suppose that we are at the $m$-th step of the search, and denote the belief before observing $\bar{y}$ by $p_{m-1}(\hat{\vx}) = \mathcal{N}(\hat{\vx}; \vmu_{m-1}, \mSigma_{m-1})$.
Ideally, we would like to update it using
\begin{align*}
p^\star_m(\hat{\vx}) \propto p(y = \bar{y} \mid \vx_i, \vx_j, \hat{\vx}) \cdot p_{m-1}(\hat{\vx}).
\end{align*}
However, this distribution is no longer Gaussian.
Therefore, we approximate it by the ``closest'' Gaussian distribution $p_m(\hat{\vx}) \doteq \mathcal{N}(\hat{\vx} ; \vmu_m, \mSigma_m)$.
We use a method known as \emph{assumed density filtering} \cite{minka2001family}, which solves the following program:
\begin{align*}
\min_{\vmu_m, \mSigma_m} \KL \left[ p^\star_m(\hat{\vx}) \ \Vert \ p_m(\hat{\vx}) \right],
\end{align*}
where $\KL[p \Vert q]$ is the Kullback-Leibler divergence from $p$ to $q$.
This can be done in closed form by computing the first two moments of the distribution $p^\star_m$, with running time $O(d^2)$.
Formally, at each step of the search, we perform the following computations (referred to further as the \text{\textsc{Update}} routine):
\begin{gather*}
  \begin{aligned}
  \mSigma_m &= \left( \mSigma_{m-1}^{-1} + \tau \vw \vw^T \right)^{-1} , \\
  \vmu_m    &= \mSigma_m [\mSigma_{m-1}^{-1} \vmu_{m-1} + (\nu - b \tau)\vw],
  \end{aligned} \\[3pt]
  \tau     = \frac{-\beta}{1+ \beta \vw^T \mSigma_{m-1} \vw} , \quad
  \nu      = \frac{\alpha - \beta(\vmu_{m-1}^T \vw + b)}{1+ \beta\vw^T \mSigma_{m-1} \vw},
  % 
  % \alpha &= \frac{\partial}{\partial \mu_{\text{proj}}} \log \Phi(\mu, \vw_{i_m, j_m}\mSigma_m\vw_{i_m, j_m}^T ) & \beta= \frac{\partial^2}{\partial \mu^2} \log \Phi(\mu, \vw_{i_m, j_m}\mSigma_m\vw_{i_m, j_m}^T) \\
\end{gather*}
where $\alpha$ and $\beta$ are the first and second derivatives with respect to $\mu = \vmu_{m-1}^T \vw_{ij}$ of the log-marginal likelihood
\begin{align*}
\log \int_{\R^d} p(\bar{y}) p(\hat{\vx}) d\hat{\vx}
    = \log \Phi \left( \frac{\vmu_{m-1}^T \vw + b}{\sqrt{ \vw^T \mSigma_{m-1} \vw + \sigma_\varepsilon^2}} \right).
\end{align*}

We summarize the full active search procedure in Algorithm~\ref{alg:search}.

% \vspace{\baselineskip}
% \vspace{-\abovedisplayskip}
\vspace{-\parskip}
% \vspace{-4ex}
% \noindent{}
% \begin{minipage}[t]{0.40\textwidth}
% \end{minipage}
% \hfill
% \begin{minipage}[t]{0.57\textwidth}
\begin{algorithm}[ht]
  % \centering
  \caption{\textsc{GaussSearch}}\label{alg:search}
  \begin{algorithmic}[1]
    \REQUIRE Objects $[n]$, feature vectors $\mathcal{X}$.
    \ENSURE target object $t$
    \STATE Initialize $\vmu_0$ and $\mSigma_1$
    \STATE $\mathcal{U} \gets \emptyset$
    \STATE $m \gets 1$
    \REPEAT
      \STATE $\vx_i, \vx_j \gets \textsc{SampleMirror}(\vmu_{m-1} , \mSigma_{m-1}, \mathcal{X}, \mathcal{U})$
      \STATE $\mathcal{U} \gets \mathcal{U} \cup \{i, j\}$
      \STATE $\bar{y} \gets$ noisy comparison outcome as per~(\ref{eq:probit})
      \STATE $\vmu_m, \mSigma_m \gets \textsc{Update}(\vx_i, \vx_j, \bar{y}, \vmu_{m-1} , \mSigma_{m-1}$)
      \STATE $m \gets m + 1$
    \UNTIL $t \in \{i,j\}$
  \end{algorithmic}
\end{algorithm}
% \end{minipage}

% \smallskip

\subsection{Convergence of \textsc{GaussSearch}}

For finite $n$, \textsc{GaussSearch} is guaranteed to terminate because samples are used without repetition.
However, in a scenario where $n$ is effectively infinite because the feature space is dense, i.e., $\mathcal{X} = \R^d$, we are able to show a much stronger result.
The following theorem asserts that, for any initial Gaussian prior distribution over the target with full-rank covariance matrix, the \textsc{GaussSearch} posterior asymptotically concentrates at the target.
\begin{theorem}
\label{thm:convergence}
  If the answers follow~(\ref{eq:probit}), then for any initial $\bm{\mu}_0$ and $\bm{\Sigma}_0 \succ 0$, as $m \to \infty$ we have
  \begin{enumerate}[label={(\roman*)}]
    \setlength\itemsep{0em}
    \setlength\parsep{0pt}
    \item $\Tr(\mSigma_m) \to 0$,
    \item $\vmu_m \to \vx_t$
  \end{enumerate}
  almost surely.
\end{theorem}

\begin{proof}[Proof sketch]
We obtain crucial insights on the asymptotic behavior by studying the case $d = 1$, when $\vx_t$, $\vmu_m$ and $\mSigma_m$ become scalars $x_t$, $\mu_m$ and $\sigma^2_m$, respectively.
We first show that $\sigma^2_m = \Theta(1 /m)$.
Next, we are able to recast $\{\mu_m\}_{m = 0}^\infty$ as a random walk biased towards $x_t$, with step size $\sigma^2_m$.
By drawing on results from stochastic approximation theory \cite{robbins1951stochastic, blum1954approximation}, we can show that $\mu_m \to x_t$ almost surely.
The extension to $d > 1$ follows by noticing that, in the dense case $\mathcal{X} = \R^d$, we can assume (without loss of generality) that the search sequentially iterates over each dimension in a round-robin fashion.
\end{proof}

This result is not a priori obvious, especially in light of the model for the oracle: answers become coin flips as the sample points $\vx_i$ and $\vx_j$ get closer.
Nevertheless, Theorem \ref{thm:convergence} guarantees that the algorithm continues making progress towards the target.

\section{Search with Latent Features}
\label{sec:embedding}

%TODO[Matt]: We have a notation issue: we overload mu and Sigma as both the prior/posterior during a search, and the distr parameters of item embeddings -> proposal: use \nu and \Phi (or sth) for the latter; also, slightly less problematic, we use z for both samples from the posterior, and below in the ELBO def

%In our second setup we assume that the objects' features $\mX$ are not accesible by us anymore, and only the oracle sees them.
In this section, we relax the assumption that the feature vectors $\mathcal{X}$ are available; instead, we consider the blind scenario, where $\mathcal{X}$ is only used implicitly by the oracle to generate answers to queries.
As a consequence, we need to generate an estimated feature embedding $\hat{\mathcal{X}} = \{\hat{\vx}_1, \ldots, \hat{\vx}_n\}$ from comparison data collected from past searches.
As we collect more data, the embedding $\hat{\mathcal{X}}$ should gradually approximate the true embedding $\mathcal{X}$ (up to rotation and reflection), improving the performance of future searches.

In a given dataset of triplet comparison outcomes, one item might be included in few triplets while another might be included in many (e.g., due to the latter being a popular search target), thus we could have more information about one item and less information about another. This suggests that when approximating $\mathcal{X}$, an \emph{uncertainty} about the item's position should be taken into account. The latest should also affect the operation of the search algorithm: After observing the outcome of oracle's comparison between two objects $i$ and $j$, whose positions $\hat{\vx}_i$ and $\hat{\vx}_j$ in our estimated embedding space are noisy, the target posterior distribution should have higher entropy than if $\hat{\vx}_i$ and $\hat{\vx}_j$ were exact (as we had assumed so far).

We incorporate this dependence by generating a distributional embedding, where each estimate $\hat{\vx}_i$ is endowed with a distribution, thus capturing the uncertainty about its true position $\vx_i$.
We also slightly modify the posterior update step (line 8 in Algorithm~\ref{alg:search}) in the search algorithm, in order to reflect the embedding uncertainty in the belief about the target.
We refer to this combined learning framework as \text{\textsc{Learn2Search}}.
It combines \text{\textsc{GaussSearch}}, our search method from the previous section, with a new triplet embedding technique called \text{\textsc{GaussEmbed}}.
We now describe our distributional embedding method, then give a brief description of the combined framework.

% \begin{algorithm}
%   \caption{\textsc{learn2search}}
%   \label{alg:engine}
%   \begin{algorithmic}[1]
%     \State Initialize $\hat{\mX} \in \R^{n \times d}$
%     \State Initialize $\mathcal{T} \gets \emptyset$
%     \While {true}
%       \State Run $k$ interactive searches using embedding $\hat{\mX}$
%       \State Collect \emph{triplets} comparisons from these searches $\mathcal{T}_{\hat{\mX}} = \{ (i_m, j_m, t_m) \}$
%       \State Aggregate $\mathcal{T} \gets \mathcal{T} \cup \mathcal{T}_{\hat{\mX}}$
%       \State Update $\hat{\mX}$ on data $\mathcal{T}$
%     \EndWhile
%   \end{algorithmic}
% \end{algorithm}

%In our second setup we assume that the objects' features $\mX$ are not accesible by us anymore, and only the oracle sees them. In this section we introduce our reinforcement learning framework \text{\textsc{Learn2Search}} that combines our search method \text{\textsc{GaussSearch}} from the previous section with a new triplet embedding technique called \text{\textsc{GaussEmb}}. Let us describe our embedding method first.

\subsection{Distributional Embedding}

%We address the problem of object embedding based on their triplet comparisons data.
As before, suppose we are given a set of objects $[n] = \{1,2,\dots,n \}$, known to have representations $\mathcal{X} \subset \R^d$ in a hidden feature space.
Although we do not have access to $\mathcal{X}$, we observe a set of noisy triplet-based relative similarities of these objects:
% \begin{align*}
  $\mathcal{T} = \{ (i,j;k) \mid \text{object $i$ is closer to $k$ than $j$ is} \},$
% \end{align*}
obtained w.r.t. probabilistic model (\ref{eq:probit}). 
% We take a bayesian viewpoint and aim to find an approximate posterior distribution over the items' embeddings, given the triplets $\mathcal{T}$. For computational tractability, we restrict ourselves to an mean-field approximate posterior, where $p(x1, ..., x_n) = \prod_i \mathcal{N}(mu_i, Sigma_i)$
The goal is to learn a $d$-dimensional Gaussian distribution embedding
\begin{align*}
q(\hat{\vx}_1, \hat{\vx}_2, \dots, \hat{\vx}_n) \doteq \prod_i \mathcal{N}(\hat{\vx}_i ; \vnu_i, \mPsi_i)
\end{align*}
where $\vnu_i \in \R^d$, $\mPsi_i \in \R^{d \times d}$ and diagonal, such that objects that are similar w.r.t. their ``true'' representations $\{\vx_i\}_{i=1}^n$ (reflected by $\mathcal{T}$) are also similar in the learned embedding.
Here for each object, $\vnu_i$ represents the mean guess on the location of object $i$ in $\R^d$, and $\mPsi_i$ represents the uncertainty about its position.

We learn this Gaussian embedding via maximizing the Evidence Lower Bound (ELBO),
\begin{align}
    \mathcal{L} = \E_q \left[ \log{p(\mathcal{T}, \hat{\vx}_1, \ldots, \hat{\vx}_n)} - \log{q(\hat{\vx}_1, \ldots, \hat{\vx}_n)} \right],\label{eq:elbo}
\end{align}
where the joint distribution $p(\mathcal{T}, \hat{\vx}_1, \ldots, \hat{\vx}_n)$ is the product of the prior $p(\hat{\vx}_1, \ldots, \hat{\vx}_n) = \prod_i \mathcal{N}(\hat{\vx}_i; \bm{0}, \mI)$ and the likelihood  $p(\mathcal{T} \mid \hat{\vx}_1, \ldots, \hat{\vx}_n)$ given w.r.t. the probabilistic model (\ref{eq:probit}) with $\sigma_\varepsilon = 1$.
We optimize the ELBO by using stochastic backpropagation with the reparametrization trick \citep{kingma2014auto, rezende2014stochastic}.
The total complexity of executing one epoch is $O(|\mathcal{T}| d)$.

% \vspace{-\abovedisplayskip}
% \begin{wrapfigure}{R}{0.5\textwidth}
% \begin{minipage}[t]{0.40\textwidth}
% \begin{algorithm}[H]
%   \caption{\textsc{learn2search}}
%   \label{alg:engine}
%   \begin{algorithmic}[1]
%     \State Initialize $\hat{\mX}$, $\hat{\vx}_i = \mathcal{N}(\vnu_i, \mPsi_i)$
%     \State Initialize $\mathcal{T} \gets \emptyset$
%     \While {true}
%       \State Run $k$ interactive searches with \text{\textsc{GaussSearch}} using embedding $\hat{\mX}$
%       \State Collect \emph{triplets} comparisons from these searches $\mathcal{T}_{\hat{\mX}} = \{ (i_m, j_m, t_m) \}$
%       \State Aggregate $\mathcal{T} \gets \mathcal{T} \cup \mathcal{T}_{\hat{\mX}}$
%       \State Update $\hat{\mX}$ on data $\mathcal{T}$ via minimizing (\ref{eq:elbo}).
%     \EndWhile
%   \end{algorithmic}
% \end{algorithm}
% \end{minipage}
% \end{wrapfigure}

\subsection{Learning to Search}

We now describe our joint embedding and search framework, \textsc{Learn2Search}.
We begin by initializing $\hat{\mathcal{X}}$ to its prior distribution, $\prod_i \mathcal{N}(\hat{\vx}_i; \bm{0}, \mI)$.
% , and is periodically updated after collecting comparison data from users through searches.
Then we alternate between searching and updating the embedding, based on all the comparison triplets collected in previous searches, by
\begin{enuminline}
\item running \text{\textsc{GaussSearch}} on the estimated embedding $\hat{\mathcal{X}}$, and
\item running \text{\textsc{GaussEmbed}} to refine $\hat{\mathcal{X}}$ based on all triplets observed so far.
\end{enuminline}
%, that approximates their actualy hidden features $\mX$ in a way that preserves the relative objects similarities, obtained during previous searches.
%One way to run \text{\textsc{GaussSearch}} "as is" on the learned $\hat{\mX}$ is to use only points' mean vectors, $\hat{\mX_\mu}$, as their feature vectors. However, we would also like to make use of the uncertainty about the position of the points given by \text{\textsc{GaussEmb}}. In order to take the full advantage of the Gaussian embedding, we incorporate the learned points' variances by enhancing our search algorithm with two modifications:
The distributional embedding requires the following modifications in \textsc{GaussSearch}.
\begin{enumerate}
\setlength\itemsep{0em}
\item in \text{\textsc{SampleMirror}}, we use the Mahalanobis distance $(\vnu - \vz)^T\mPsi^{-1}(\vnu - \vz)$ when finding the closest objects $i, j$ to the sampled points $\vz_1, \vz_2$.

\item in the \text{\textsc{Update}} step, we combine the variance of the noisy outcome $\sigma^2_\varepsilon$ with the uncertainty over the location of $\hat{\vx}_i$ and $\hat{\vx}_j$ along the direction perpendicular to the hyperplane, resulting in an effective noise variance
$\sigma_{\text{eff}}^2 = \sigma_\varepsilon^2 + \vw^T \mPsi_i \vw + \vw^T \mPsi_j \vw$. 
\end{enumerate}
Both modifications results in a natural tradeoff between exploration and exploitation: At first, we tend to sample uncertain points more frequently, and the large effective noise variance restricts the speed at which the covariance of the belief decreases.
After multiple searches, as we get more confident about the embedding, we start exploiting the structure of the space more assertively.
% We sum up the general scheme of \text{\textsc{Learn2Search}} in Algorithm~\ref{alg:engine}.
%% \filbreak
%% \section{Experiments}
%% \filbreak
%% \label{sec:experiments}
%% \filbreak
%% \subsection{\text{\textsc{GaussSearch}}: Query And Computational Complexities}
%% \filbreak

% \noindent{}

% \begin{algorithm}
%   \caption{\textsc{learn2search}}
%   \label{alg:engine}
%   \begin{algorithmic}[1]
%     \State Initialize $\hat{\mX} \in \R^{n \times d}$
%     \State Initialize $\mathcal{T} \gets \emptyset$
%     \While {true}
%       \State Run $k$ interactive searches using embedding $\hat{\mX}$
%       \State Collect \emph{triplets} comparisons from these searches $\mathcal{T}_{\hat{\mX}} = \{ (i_m, j_m, t_m) \}$
%       \State Aggregate $\mathcal{T} \gets \mathcal{T} \cup \mathcal{T}_{\hat{\mX}}$
%       \State Update $\hat{\mX}$ on data $\mathcal{T}$
%     \EndWhile
%   \end{algorithmic}
% \end{algorithm}

\section{Experiments}
\label{sec:experiments}

In this section, we evaluate our approach empirically and compare it to competing methods.
In Section~\ref{sec:exp:nonblind} we focus on the query complexity and computational performance of comparison-based search algorithms under the assumption that the feature vectors $\mathcal{X}$ are known.
We then turn to the blind setting. In Section~\ref{sec:exp:blind} we evaluate our joint search \& embedding framework on two real world datasets and compare it against state-of-the-art embedding methods. In Section~\ref{sec:fs-experiment} we describe the results of applying our framework in an experiment on searching for movie actors involving real users.

\subsection{Non-Blind Setting}
\label{sec:exp:nonblind}

\begin{figure*}[ht]
\centering
\vspace{-3mm}
\subfloat[Query complexity]{\label{search_q}\includegraphics[width=8.5cm]{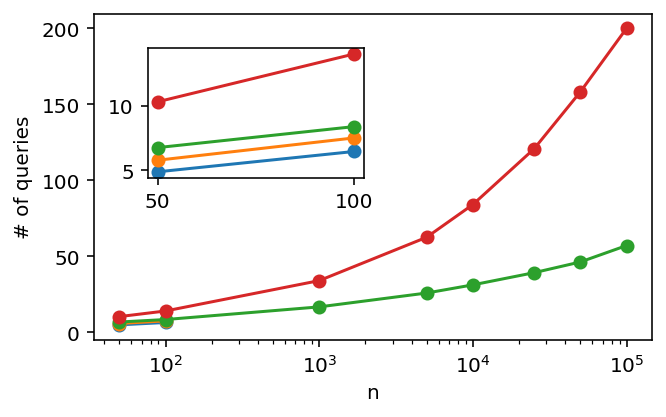}} 
\subfloat[Computational complexity]{\label{search_t}\includegraphics[width=8.5cm]{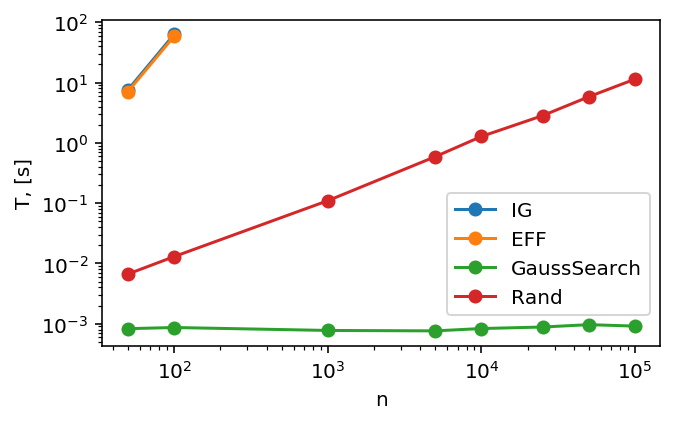}}
\caption{Average (a) query complexity and (b) computational complexity of four search algorithms in the non-blind setting for $d=5$ and increasing $n$.
Search success is declared when $\arg\max_{i \in [n]}P(\vx_i) = t$.}
\label{search_plot}
% \vspace{-4mm}
\end{figure*}

The main difference between \textsc{GaussSearch} and previous methods is the way the uncertainty about the target object is modeled.
As outlined in Section~\ref{sec:relwork}, all previous approaches consider a discrete posterior distribution $P = \{p_1, p_2, \dots, p_n\}$ over all objects from $[n]$, which is updated after each step using Bayes' rule.
We compared \textsc{GaussSearch} to the baselines that operate on the full discrete distribution and that use different rules to choose the next query:
(1) \textsc{IG}: the next query $(i,j)$ is chosen to maximize the expected information gain,
(2) \textsc{Eff}: a fast approximation of the \textsc{EC$^2$} active learning algorithm \cite{golovin2010near} and
(3) \textsc{Rand}: the pair of query points is chosen uniformly at random from $[n]$.

In order to assess how well \textsc{GaussSearch} scales in terms of both query and computational complexities in comparison to the other baselines, we run simulations on synthetically generated data of $n$ points uniformly sampled from a $d=5$ dimensional hypercube.
We vary the number of points $n$ from 50 to 100000.
For each value of $n$, we run 1000 searches, each with a new target sampled independently and uniformly at random among all objects.
During the search, comparison outcomes are sampled from the probit model~(\ref{eq:probit}), using a value of $\sigma_\varepsilon^2$ chosen such that approximately 10\% of the queries' outcomes are flipped.
In order not to give an undue advantage to our algorithm (which never uses an object in a comparison pair more than once), we change the stopping criterion, and declare that a search is successful as soon as the target $t$ becomes the point with the highest probability mass under a given method's target model $P$ (ties are broken at random).
For the \textsc{GaussSearch} procedure, we take $P(\vx_i)$ to be proportional to the density of the Gaussian posterior at $\vx_i$.
% transformed Gaussian posterior on each step into a set of pointwise probabilities by weighing each object by its corresponding density under the current Gaussian distribution and then normalizing the resulting values. 
We use two performance metrics.
\begin{description}
\setlength\parskip{0pt}
\setlength\parsep{0pt}
\item[Query complexity]
The average number of queries until $\arg\max_{i \in [n]}P(\vx_i) = t$, i.e., the true target point has the highest posterior probability.

\item[Computational complexity]
The total time needed for an algorithm to decide which query to make next and to update the posterior upon receiving a comparison outcome.
\end{description}
Fig~\ref{search_q} and Fig~\ref{search_t} show the results averaged over 1000 experiments.
Both \textsc{IG} and \textsc{Eff} require $O(n^3)$ operations to choose the next query, as they perform a greedy search over all possible combinations of query pair and target.
For these two methods, we report results only for $n = 50$ and $n = 100$, as the time required for finding the optimal query and updating the posterior for $n > 100$ takes over a minute per one step.
For $n \in \{50, 100\}$, \textsc{IG} performs best in terms of query complexity, with 4.9 and 6.47 queries per search on average for $n = 50$ and $n=100$, respectively.
However, the running time is prohibitively long already for $n=100$.
In constrast, \textsc{GaussSearch} has a favorable tradeoff between query complexity and computational complexity: in comparison to \textsc{IG}, it requires only 1.8 (for $n=50$) or 2 (for $n=100$) additional queries on average, while the running time is improved by several orders of magnitude.
We observe that, even though in theory the running time of \textsc{GaussSearch} is not independent of $n$, in practice it is almost constant throughout the range of values of $n$ that we investigate.
Finally, we note that \textsc{Rand} performs worst in terms of average number of queries, despite having a higher running time per search step (due to maintaining a discrete posterior).
% 
%Indeed, both \textsc{IG} and \textsc{Eff} require $O(n^3)$ operations to choose the next query, as they perform a greedy search over all possible combinations of query pair and target.
%Furthermore, as all three baselines store full points distribution, they all have $O(n)$ posterior update time complexity and $O(n)$ space complexity.
%This is in contrast to \textsc{GaussSearch} that finds query in $O(\log{n})$ time, updates the posterior in $O(d^2)$ and has $O(d^3)$ space complexity.

% \vspace{-6mm}
\subsection{Blind Setting}
\label{sec:exp:blind}

% \begin{figure}[ht]
% \vspace{-3mm}
% \begin{center}
% \includegraphics[width=8.5cm]{plots/sande.png}
% \end{center}
% \vspace{-6mm}
% \caption{Combined search and embedding framework \textsc{Learn2Search} on two datasets with different embedding techniques.}
% \label{sande_plot}
% \vspace{-5mm}
% \end{figure}

\begin{figure*}[htp]
\centering
% \subfloat{\label{embedding_music}\includegraphics[width=8.5cm, height=5.5cm]{plots/sande_w1.png}}
% \subfloat{\label{embedding_food}\includegraphics[width=8.5cm, height=5.5cm]{plots/sande_m1.png}}
\subfloat{\label{embedding_music}\includegraphics[width=8.5cm]{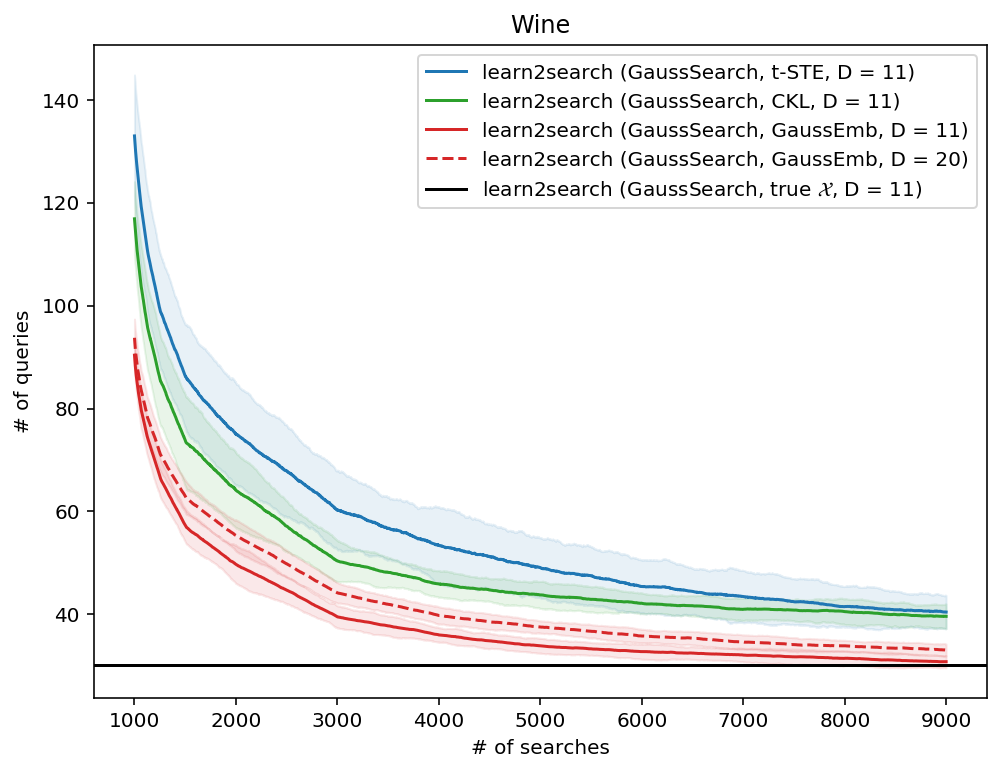}}
\subfloat{\label{embedding_food}\includegraphics[width=8.5cm]{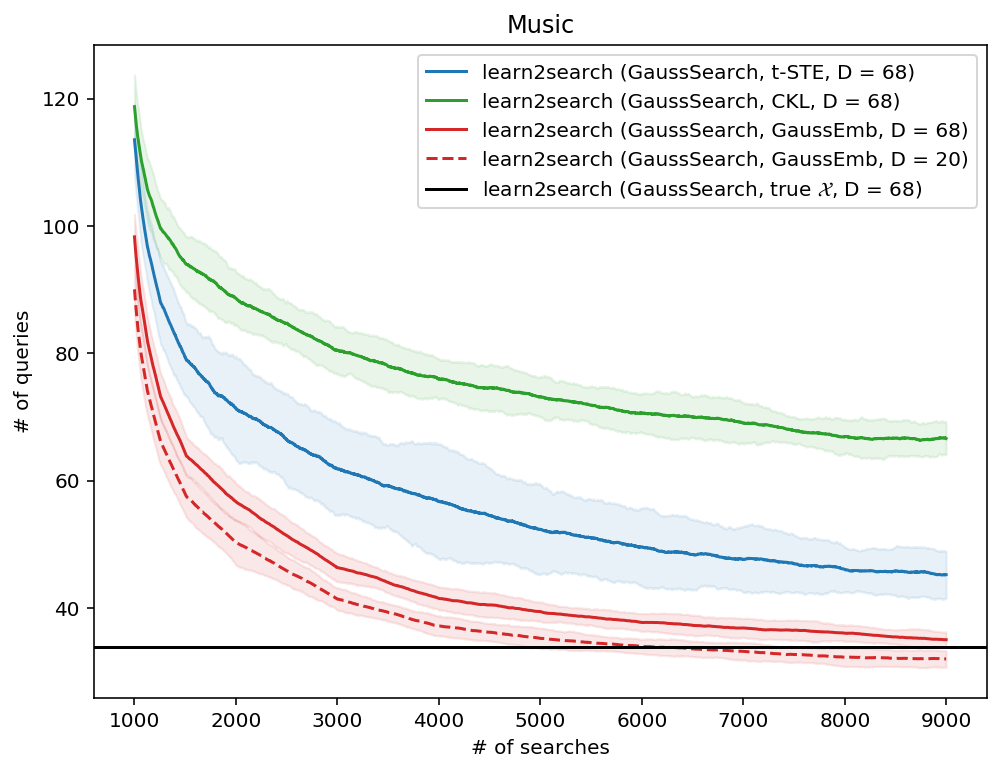}}
\caption{Combined search and embedding framework \textsc{learn2search} on two datasets with different embedding techniques and choices of $D$. Results are reported over a sliding window of size 1000.}
\label{sande_plot}
\end{figure*}

\begin{figure*}[ht]
\centering
\vspace{-3mm}
% \subfloat[]{\label{fig:fs_searchcost}\includegraphics[width=8.5cm, height=5.5cm]{plots/actors_exp2_q.png}} 
% \subfloat[]{\label{fig:m_dist}\includegraphics[width=8.5cm, height=5.5cm]{plots/actor_triplets.png}}
\subfloat[]{\label{fig:fs_searchcost}\includegraphics[width=8.5cm]{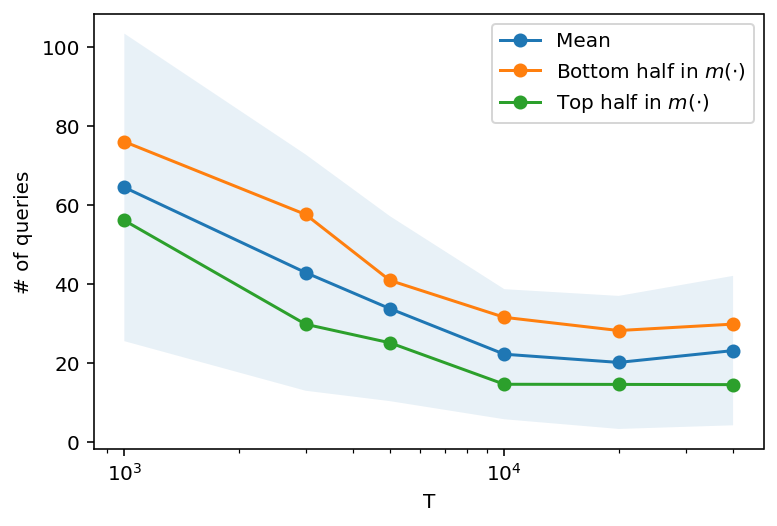}} 
\subfloat[]{\label{fig:m_dist}\includegraphics[width=8.5cm]{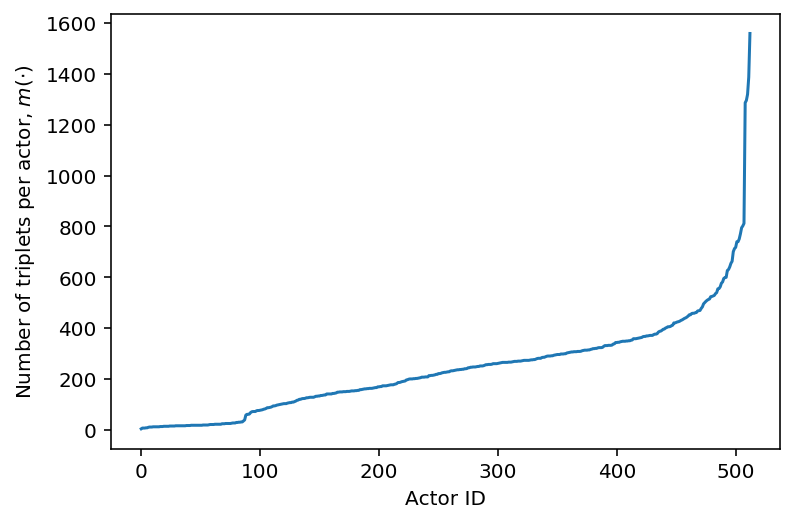}}
\caption{(a) Empirical search cost, averaged over $5$ human subjects, for a database with $n=513$ faces of movie actors. The averages in the top and bottom group of targets (in terms of number of triplets) are shown individually as well: clearly, the denser group (green) has lower search cost. The blue shaded area shows one std around the average search cost. Note that the $y$-axis is the number of queries (each query features four pictures). Each choice is broken down into 3 pairwise comparison outcomes. (b) Distribution of the number of collected triplets per actor, $m(\cdot)$. The actor with the highest number of triplets is Halle Berry with $m(\mbox{Halle Berry})=1567$ and the actor with the least number of triplets is Kumail Nanjiani with $m(\mbox{Kumail Nanjiani})=16$. The median number of triplets per actor is 227.}
% \label{search_plot}
% \vspace{-4mm}
\end{figure*}

We now turn our attention to the blind setting and study the performance of \textsc{Learn2Search} on two real world datasets:
the \emph{red wine} dataset~\cite{cortez2009modeling} and the \emph{music} dataset~\cite{zhou2014predicting}.
These datasets were studied in the context of comparison-based search by \citet{kazemi2018comparison}.
We assume that the feature vectors $\{\vx_1, \ldots, \vx_n\}$ are latent: they are used to generate comparison outcomes, but they are not available to the search algorithm.
We compare different approaches to learning an estimate of the feature vectors from comparison outcomes.
In addition to our distributional embedding method \textsc{GaussEmbed}, we consider two state-of-the-art embedding algorithms:
$t$-STE \cite{van2012stochastic} and CKL \cite{tamuz2011adaptively} (both of which produce point estimates of the feature vectors), as well as the ground-truth vectors $\mathcal{X}$.

We run 9000 searches in total, starting from a randomly initialized $\hat{\mathcal{X}}$ and updating the embedding on the $k$-th search for $k \in [2^0, 2^1, \ldots, 2^{13}]$
For each search, the target object is chosen uniformly at random among all objects.
The outcomes of the comparisons queried by \textsc{GaussSearch} are sampled from the probit model~(\ref{eq:probit}) based on the ground-truth feature vectors $\mathcal{X}$.
A search episode ends when the target appears as one of the two objects in the query.
As before, the noise level $\sigma_\varepsilon^2$ is set to corrupt approximately 10\% of the answers on average.
As we jointly perform searches and learn the embedding, we measure the number of queries needed for \textsc{GaussSearch} to find the target using the current version of $\hat{\mathcal{X}}$ for each search, and then we average these numbers over a moving window of size 1000.
% Thus the resulting first value is the average number of queries in the first 1000 searches, the second value is the average number of queries in the searches $2,\dots,1001$ and so on.

We present the results, averaged over 100 experiments, in Fig.~\ref{sande_plot}. The combination of \textsc{GaussSearch} and \textsc{GaussEmbed} manages to learn object representations that give rise to search episodes that are as query-efficient as they would have been using the ground-truth vectors $\mathcal{X}$, and significantly outperforms variants using $t$-STE and CKL for generating embeddings.
It appears that taking into account the uncertainty over the points' locations, as is done by \textsc{GaussEmbed}, helps to make fewer mistakes in early searches and thus leads to a lower query complexity in comparison to using point-estimates of vector embeddings.
As the number of searches grows, \textsc{GaussEmbed} learns better and better representations $\hat{\mathcal{X}}$ and enables \textsc{GaussSearch} to ask fewer and fewer queries in order to find the target.
In the final stages, our blind search method is as efficient as if it had access to the true embedding on both \emph{wine} and \emph{music} datasets. 

One of the challenges that arises when learning an embedding in the blind setting is the choice of $D$. In the case when the true $d$ is not known, we estimated it by first conservatively setting $D=100$, and then, after collecting around 10000 triplets, picking the smallest value of $D$ with 98\% of the energy in the eigenvalue decomposition of the covariance matrix of the mean vectors $\hat{\mX_\mu}$. This number for both datasets was $20$, and $D=20$ was used as the approximation of the actual $d$.
The results of running our \textsc{Learn2Search} with the estimated $d$ is shown in Fig.~\ref{sande_plot} as the dashed lines.
For the wine dataset, we achieved almost the same query complexity as for true $d$.
For the music dataset, after 4000 searches our scheme with estimated $d$ actually outperformed the true features $\mX$: on average, when the search is run on $\hat{\mX}$, \textsc{GaussSearch} needs fewer queries than when it is performed on $\mX$.
This suggests that \textsc{Learn2Search} is not only robust to the choice of $D$, but is also capable of learning useful object representations for the comparison-based search independent of the true features of the objects.
% This suggests that \textsc{Learn2Search} is capable of learning object representations that are useful for the comparison-based search procedure independently of the objects' true feature vectors.

\subsection{Movie Actors Face Search Experiment}
\label{sec:fs-experiment}

We present the results of an experiment involving human oracles, in order to validate the practical relevance of the model and search algorithm developed in this paper.
We collected $n=513$ photographs of the faces of well-known actors and actresses, and we built a web-based search interface implementing the \textsc{Learn2Search} method over this dataset.\footnote{Movie Actors Search: \url{http://who-is-th.at/}. The source code for our methods as well as the data collected in the face search experiment will be available at \url{https://indy.epfl.ch/}.}
In this system, a user can look for an actor by comparing faces: at each search iteration, the user is shown $4$ pictures of faces, and can then click on the face that looks the most like the target; this process repeats until the target is among the $4$ samples.
Once it is, the user clicks on a button to see details about the target, which ends the search.
We have collected slightly more than $\num{40000}$ comparison triplets from users of the site (which include both project participants as well as outside users, who explored the service out of curiosity).
The system is \emph{blind}, in that we do not extract any explicit features of the faces (shape, eye color, etc.)

We try to answer the following two questions:
(i) how efficient are searches? and (ii) how does the search cost depend on the size of the training set?
For this, we recruited $5$ subjects, whose task it was to find the face of an actor, which has been sampled uniformly from the $n$ possible targets.
For each search, we first sampled $T \in \{1000, 3000, 5000, 10k, 20k, 40k\}$ triplets, uniformly from the source $40k+$ triplets, and then generated the embedding as described in Section~\ref{sec:embedding}.
For each value of $T$, each subject performed $10$ searches.
The results are shown in Fig.~\ref{fig:fs_searchcost}.
Observe that for the smallest $T$, the average search cost is essentially equal to the random strategy (which has expected cost $\frac{1}{2} \cdot \frac{n}{4}$).
As $T$ increases, the embedding becomes more meaningful, and the search cost drops significantly. In our experiments, the median time-to-answer is around 5 seconds. Thus, once we have \char`\~10k triplets, the time until the target is found is \char`\~2 minutes. We observe small differences in the time-to-answer as a function of the quality of the embedding: with random embeddings, users tend to answer slightly faster (probably because the images are very different to the target).

An additional comment on the variance in search cost is in order.
Call $m(i)$ the number of triplets in $T$ that object $i$ is part of (in any position).
In our full experimental dataset, $m(i)$ is heavily skewed, see Fig.~\ref{fig:m_dist}, because
\begin{enuminline}
\item objects were added to the system gradually over time, and
\item the distribution over targets follows popularity, which is heavily skewed itself.
\end{enuminline}
Concretely, this ranges from $m(\mbox{Halle Berry})=1567$ to $m(\mbox{Kumail Nanjiani})=16$.
It is natural to suspect that a target $t$ with larger $m(t)$ is easier to find, because its embedding is more precise relative to other objects.
Indeed, our results bear this out: in Fig.~\ref{fig:fs_searchcost}, we break out the search cost for the top and bottom half of targets separately.
This suggests that if we could collect further data, including for the sparse targets (low $m(t)$), the asymptotic search cost would be reduced further.
In summary, our experiment shows that \textsc{Learn2Search} is able to extract an embedding in the blind setting that appears to align with visual features that human oracles rely on to find a face, and is able to navigate through this embedding to locate a target efficiently.

\section{Conclusion}
\label{sec:conclusion}

In this work, we introduce
%\begin{enuminline}
a probablistic model of triplet comparisons,
a search algorithm for finding a target among a set of alternatives, and
a Bayesian embedding algorithm for triplet data.
%\end{enuminline}
Our search algorithm is backed by theoretical guarantees and strikes a favorable computational vs. query complexity trade-off.
By combining the search and embedding algorithms, we design a system that learns to search efficiently without a-priori access to object features.
We observe that the search cost successfully decreases as the number of search episodes increases.
This is achieved by iteratively refining a distributional, latent, and low-dimensional representation of the objects.
Our framework is scalable in the number of objects $n$, tolerates noisy answers, and performs well on synthetic and real-world experiments.

\section*{Acknowledgements}

The authors wish to thank Olivier Cloux for the development of the \url{https://who-is-th.at/} platform, and for his help with running and evaluating the experiments. We are grateful to Holly Cogliati for proof-reading the manuscript.

\balance
\bibliography{learn2search}

\begin{thebibliography}{4}
\providecommand{\natexlab}[1]{#1}
\providecommand{\url}[1]{\texttt{#1}}
\expandafter\ifx\csname urlstyle\endcsname\relax
  \providecommand{\doi}[1]{doi: #1}\else
  \providecommand{\doi}{doi: \begingroup \urlstyle{rm}\Url}\fi

\bibitem[Blum et~al.(1954)]{blum1954approximation}
Blum, J.~R. et~al.
\newblock Approximation methods which converge with probability one.
\newblock \emph{The Annals of Mathematical Statistics}, 25\penalty0
  (2):\penalty0 382--386, 1954.

\bibitem[Ellis et~al.(2002)Ellis, Whitman, Berenzweig, and
  Lawrence]{ellis2002quest}
Ellis, D.~P., Whitman, B., Berenzweig, A., and Lawrence, S.
\newblock The quest for ground truth in musical artist similarity.
\newblock In \emph{ISMIR}. Paris, France, 2002.

\bibitem[Robbins \& Monro(1951)Robbins and Monro]{robbins1951stochastic}
Robbins, H. and Monro, S.
\newblock A stochastic approximation method.
\newblock \emph{The annals of mathematical statistics}, 22\penalty0
  (3):\penalty0 400--407, 1951.

\bibitem[Wilber et~al.(2014)Wilber, Kwak, and Belongie]{wilber2014cost}
Wilber, M.~J., Kwak, I.~S., and Belongie, S.~J.
\newblock Cost-effective hits for relative similarity comparisons.
\newblock In \emph{Second AAAI conference on human computation and
  crowdsourcing}, 2014.

\end{thebibliography}


\begin{thebibliography}{43}
\providecommand{\natexlab}[1]{#1}
\providecommand{\url}[1]{\texttt{#1}}
\expandafter\ifx\csname urlstyle\endcsname\relax
  \providecommand{\doi}[1]{doi: #1}\else
  \providecommand{\doi}{doi: \begingroup \urlstyle{rm}\Url}\fi

\bibitem[Agarwal et~al.(2007)Agarwal, Wills, Cayton, Lanckriet, Kriegman, and
  Belongie]{agarwal2007generalized}
Agarwal, S., Wills, J., Cayton, L., Lanckriet, G., Kriegman, D., and Belongie,
  S.
\newblock Generalized non-metric multidimensional scaling.
\newblock In \emph{Artificial Intelligence and Statistics}, pp.\  11--18, 2007.

\bibitem[Amid \& Ukkonen(2015)Amid and Ukkonen]{amid2015multiview}
Amid, E. and Ukkonen, A.
\newblock Multiview triplet embedding: Learning attributes in multiple maps.
\newblock In \emph{International Conference on Machine Learning}, pp.\
  1472--1480, 2015.

\bibitem[Anderton \& Aslam(2019)Anderton and Aslam]{anderton2019scaling}
Anderton, J. and Aslam, J.
\newblock Scaling up ordinal embedding: A landmark approach.
\newblock In \emph{International Conference on Machine Learning}, pp.\
  282--290, 2019.

\bibitem[Bentley(1975)]{bentley1975multidimensional}
Bentley, J.~L.
\newblock Multidimensional binary search trees used for associative searching.
\newblock \emph{Communications of the ACM}, 18\penalty0 (9):\penalty0 509--517,
  1975.

\bibitem[Blum et~al.(1954)]{blum1954approximation}
Blum, J.~R. et~al.
\newblock Approximation methods which converge with probability one.
\newblock \emph{The Annals of Mathematical Statistics}, 25\penalty0
  (2):\penalty0 382--386, 1954.

\bibitem[Bradley \& Terry(1952)Bradley and Terry]{bradley1952rank}
Bradley, R.~A. and Terry, M.~E.
\newblock Rank analysis of incomplete block designs: {I.} {The} method of
  paired comparisons.
\newblock \emph{Biometrika}, 39\penalty0 (3/4):\penalty0 324--345, 1952.

\bibitem[Brochu et~al.(2008)Brochu, de~Freitas, and Ghosh]{brochu2008active}
Brochu, E., de~Freitas, N., and Ghosh, A.
\newblock Active preference learning with discrete choice data.
\newblock In \emph{Advances in neural information processing systems}, pp.\
  409--416, 2008.

\bibitem[Canal et~al.(2019)Canal, Massimino, Davenport, and
  Rozell]{canal2019icml}
Canal, G., Massimino, A., Davenport, M., and Rozell, C.
\newblock Active embedding search via noisy paired comparisons.
\newblock In \emph{Proceedings of the 36th International Conference on Machine
  Learning (ICML)}, 2019.

\bibitem[Chang \& Tsao(2017)Chang and Tsao]{chang2017code}
Chang, L. and Tsao, D.~Y.
\newblock The code for facial identity in the primate brain.
\newblock \emph{Cell}, 169\penalty0 (6):\penalty0 1013--1028, 2017.

\bibitem[Chu \& Ghahramani(2005)Chu and Ghahramani]{chu2005extensions}
Chu, W. and Ghahramani, Z.
\newblock Extensions of {Gaussian} processes for ranking: {Semi}-supervised and
  active learning.
\newblock In \emph{Proceedings of the NIPS 2005 Workshop on Learning to Rank},
  Whistler, BC, Canada, December 2005.

\bibitem[Cortez et~al.(2009)Cortez, Cerdeira, Almeida, Matos, and
  Reis]{cortez2009modeling}
Cortez, P., Cerdeira, A., Almeida, F., Matos, T., and Reis, J.
\newblock Modeling wine preferences by data mining from physicochemical
  properties.
\newblock \emph{Decision Support Systems}, 47\penalty0 (4):\penalty0 547--553,
  2009.

\bibitem[Cox et~al.(2000)Cox, Miller, Minka, Papathomas, and
  Yianilos]{cox2000bayesian}
Cox, I.~J., Miller, M.~L., Minka, T.~P., Papathomas, T.~V., and Yianilos, P.~N.
\newblock The bayesian image retrieval system, pichunter: theory,
  implementation, and psychophysical experiments.
\newblock \emph{IEEE transactions on image processing}, 9\penalty0
  (1):\penalty0 20--37, 2000.

\bibitem[Dasgupta(2005)]{dasgupta2005analysis}
Dasgupta, S.
\newblock Analysis of a greedy active learning strategy.
\newblock In \emph{Advances in neural information processing systems}, pp.\
  337--344, 2005.

\bibitem[Datta et~al.(2008)Datta, Joshi, Li, and Wang]{datta2008image}
Datta, R., Joshi, D., Li, J., and Wang, J.~Z.
\newblock Image retrieval: Ideas, influences, and trends of the new age.
\newblock \emph{ACM Computing Surveys (Csur)}, 40\penalty0 (2):\penalty0 5,
  2008.

\bibitem[Ellis et~al.(2002)Ellis, Whitman, Berenzweig, and
  Lawrence]{ellis2002quest}
Ellis, D.~P., Whitman, B., Berenzweig, A., and Lawrence, S.
\newblock The quest for ground truth in musical artist similarity.
\newblock In \emph{ISMIR}. Paris, France, 2002.

\bibitem[Fang \& Geman(2005)Fang and Geman]{fang2005experiments}
Fang, Y. and Geman, D.
\newblock Experiments in mental face retrieval.
\newblock In \emph{International Conference on Audio-and Video-Based Biometric
  Person Authentication}, pp.\  637--646. Springer, 2005.

\bibitem[Ferecatu \& Geman(2009)Ferecatu and Geman]{ferecatu2009statistical}
Ferecatu, M. and Geman, D.
\newblock A statistical framework for image category search from a mental
  picture.
\newblock \emph{IEEE Transactions on Pattern Analysis and Machine
  Intelligence}, 31\penalty0 (6):\penalty0 1087--1101, 2009.

\bibitem[Ghosh et~al.(2019)Ghosh, Chen, and Yue]{ghosh2019landmark}
Ghosh, N., Chen, Y., and Yue, Y.
\newblock Landmark ordinal embedding.
\newblock In \emph{Advances in Neural Information Processing Systems}, pp.\
  11506--11515, 2019.

\bibitem[Golovin et~al.(2010)Golovin, Krause, and Ray]{golovin2010near}
Golovin, D., Krause, A., and Ray, D.
\newblock Near-optimal bayesian active learning with noisy observations.
\newblock In \emph{Advances in Neural Information Processing Systems}, pp.\
  766--774, 2010.

\bibitem[Haghiri et~al.(2017)Haghiri, Ghoshdastidar, and von
  Luxburg]{haghiri2017comparison}
Haghiri, S., Ghoshdastidar, D., and von Luxburg, U.
\newblock Comparison-based nearest neighbor search.
\newblock In \emph{Artificial Intelligence and Statistics}, pp.\  851--859,
  2017.

\bibitem[Heim et~al.(2015)Heim, Berger, Seversky, and
  Hauskrecht]{heim2015efficient}
Heim, E., Berger, M., Seversky, L.~M., and Hauskrecht, M.
\newblock Efficient online relative comparison kernel learning.
\newblock In \emph{Proceedings of the 2015 SIAM International Conference on
  Data Mining}, pp.\  271--279. SIAM, 2015.

\bibitem[Houlsby et~al.(2011)Houlsby, Husz{\'a}r, Ghahramani, and
  Lengyel]{houlsby2011bayesian}
Houlsby, N., Husz{\'a}r, F., Ghahramani, Z., and Lengyel, M.
\newblock Bayesian active learning for classification and preference learning.
\newblock \emph{arXiv preprint arXiv:1112.5745}, 2011.

\bibitem[Jain et~al.(2016)Jain, Jamieson, and Nowak]{jain2016finite}
Jain, L., Jamieson, K.~G., and Nowak, R.
\newblock Finite sample prediction and recovery bounds for ordinal embedding.
\newblock In \emph{Advances In Neural Information Processing Systems}, pp.\
  2711--2719, 2016.

\bibitem[Jamieson \& Nowak(2011)Jamieson and Nowak]{jamieson2011low}
Jamieson, K.~G. and Nowak, R.~D.
\newblock Low-dimensional embedding using adaptively selected ordinal data.
\newblock In \emph{Communication, Control, and Computing (Allerton), 2011 49th
  Annual Allerton Conference on}, pp.\  1077--1084. IEEE, 2011.

\bibitem[Karbasi et~al.(2012)Karbasi, Ioannidis, and
  Massoulié]{karbasi2012comparison}
Karbasi, A., Ioannidis, S., and Massoulié, L.
\newblock Comparison-based learning with rank nets.
\newblock In \emph{Proceedings of the 29th International Conference on Machine
  Learning (ICML)}, 2012.

\bibitem[Kazemi et~al.(2018)Kazemi, Chen, Dasgupta, and
  Karbasi]{kazemi2018comparison}
Kazemi, E., Chen, L., Dasgupta, S., and Karbasi, A.
\newblock Comparison based learning from weak oracles.
\newblock In \emph{International Conference on Artificial Intelligence and
  Statistics}, pp.\  1849--1858, 2018.

\bibitem[Kingma \& Welling(2014)Kingma and Welling]{kingma2014auto}
Kingma, D.~P. and Welling, M.
\newblock Auto-encoding variational {Bayes}.
\newblock In \emph{Proceedings of ICLR 2014}, Banff, AB, Canada, April 2014.

\bibitem[Kleindessner \& von Luxburg(2017)Kleindessner and von
  Luxburg]{kleindessner2017kernel}
Kleindessner, M. and von Luxburg, U.
\newblock Kernel functions based on triplet comparisons.
\newblock In \emph{Advances in Neural Information Processing Systems}, pp.\
  6807--6817, 2017.

\bibitem[MacKay(1992)]{mackay1992information}
MacKay, D.~J.
\newblock Information-based objective functions for active data selection.
\newblock \emph{Neural computation}, 4\penalty0 (4):\penalty0 590--604, 1992.

\bibitem[Minka(2001)]{minka2001family}
Minka, T.~P.
\newblock \emph{A family of algorithms for approximate Bayesian inference}.
\newblock PhD thesis, Massachusetts Institute of Technology, 2001.

\bibitem[Nowak(2008)]{nowak2008generalized}
Nowak, R.
\newblock Generalized binary search.
\newblock In \emph{Communication, Control, and Computing, 2008 46th Annual
  Allerton Conference on}, pp.\  568--574. IEEE, 2008.

\bibitem[Nowak(2009)]{nowak2009noisy}
Nowak, R.
\newblock Noisy generalized binary search.
\newblock In \emph{Advances in neural information processing systems}, pp.\
  1366--1374, 2009.

\bibitem[Rezende et~al.(2014)Rezende, Mohamed, and
  Wierstra]{rezende2014stochastic}
Rezende, D.~J., Mohamed, S., and Wierstra, D.
\newblock Stochastic backpropagation and approximate inference in deep
  generative models.
\newblock In \emph{International Conference on Machine Learning}, pp.\
  1278--1286, 2014.

\bibitem[Robbins \& Monro(1951)Robbins and Monro]{robbins1951stochastic}
Robbins, H. and Monro, S.
\newblock A stochastic approximation method.
\newblock \emph{The annals of mathematical statistics}, 22\penalty0
  (3):\penalty0 400--407, 1951.

\bibitem[Settles(2012)]{settles2012active}
Settles, B.
\newblock \emph{Active Learning}.
\newblock Morgan \& Claypool Publishers, 2012.

\bibitem[Suditu \& Fleuret(2012)Suditu and Fleuret]{suditu2012iterative}
Suditu, N. and Fleuret, F.
\newblock Iterative relevance feedback with adaptive exploration/exploitation
  trade-off.
\newblock In \emph{Proceedings of the 21st ACM International Conference on
  Information and Knowledge Management}, pp.\  1323--1331. ACM, 2012.

\bibitem[Sun et~al.(2012)Sun, Wang, Wang, Shao, and Li]{sun2012efficient}
Sun, Z., Wang, H., Wang, H., Shao, B., and Li, J.
\newblock Efficient subgraph matching on billion node graphs.
\newblock \emph{Proceedings of the VLDB Endowment}, 5\penalty0 (9):\penalty0
  788--799, 2012.

\bibitem[Tamuz et~al.(2011)Tamuz, Liu, Belongie, Shamir, and
  Kalai]{tamuz2011adaptively}
Tamuz, O., Liu, C., Belongie, S., Shamir, O., and Kalai, A.~T.
\newblock Adaptively learning the crowd kernel.
\newblock In \emph{Proceedings of the 28th International Conference on Machine
  Learning (ICML)}, pp.\  673--680. Omnipress, 2011.

\bibitem[Thurstone(1927)]{thurstone1927law}
Thurstone, L.~L.
\newblock A law of comparative judgment.
\newblock \emph{Psychological Review}, 34\penalty0 (4):\penalty0 273--286,
  1927.

\bibitem[Tschopp et~al.(2011)Tschopp, Diggavi, Delgosha, and
  Mohajer]{tschopp2011randomized}
Tschopp, D., Diggavi, S., Delgosha, P., and Mohajer, S.
\newblock Randomized algorithms for comparison-based search.
\newblock In \emph{Advances in Neural Information Processing Systems}, pp.\
  2231--2239, 2011.

\bibitem[Van Der~Maaten \& Weinberger(2012)Van Der~Maaten and
  Weinberger]{van2012stochastic}
Van Der~Maaten, L. and Weinberger, K.
\newblock Stochastic triplet embedding.
\newblock In \emph{Machine Learning for Signal Processing (MLSP), 2012 IEEE
  International Workshop on}, pp.\  1--6. IEEE, 2012.

\bibitem[Wilber et~al.(2014)Wilber, Kwak, and Belongie]{wilber2014cost}
Wilber, M.~J., Kwak, I.~S., and Belongie, S.~J.
\newblock Cost-effective hits for relative similarity comparisons.
\newblock In \emph{Second AAAI conference on human computation and
  crowdsourcing}, 2014.

\bibitem[Zhou et~al.(2014)Zhou, Claire, and King]{zhou2014predicting}
Zhou, F., Claire, Q., and King, R.~D.
\newblock Predicting the geographical origin of music.
\newblock In \emph{2014 IEEE international conference on data mining (ICDM)},
  pp.\  1115--1120. IEEE, 2014.

\end{thebibliography}

\clearpage
% \documentclass{article}
% \input{preamble.tex}
% \input{math_commands.tex}

% % Pick-up labels from main text.
% \usepackage{xr}
% \externaldocument{learn2search}

% %\title{Learning to Search Efficiently Using Comparisons \\
% \title{Scalable and Efficient Comparison-based Search without Features
%  \\
% Supplementary Material}

% \begin{document}

% \date{}

% \maketitle

\appendix

\section{Supplementary Material}

This supplementary material contains:
\begin{enumerate}
	\item Proofs of Theorem~\ref{thm:optimal} and Theorem~\ref{thm:convergence}.
	\item Additional experiments on comparing different embedding techniques.
	\item Additional plot of the 2-PCA of the learned embedding in the movie actors face search experiment.
\end{enumerate}

\section{Proof of Theorem~\ref{thm:optimal}}

\begin{proof}
Without loss of generality, assume that $\sigma_\varepsilon = 1$.
Let $y$ be a binary random variable such that $p(y = 1 | \vw, b, \vx) = \Phi(\vx^\tr \vw + b)$.
Then,
\begin{align}
&\argmax_{(\vw, b) \in \mathcal{H}} I[\vx; y \mid (\vw, b)] \nonumber \\
&\quad = \argmax_{(\vw, b) \in \mathcal{H}} \left\{ 1 - \E_{\hat{\vx}}[H(y \mid \vw, b, \vx)] \right\} \label{eq:step1} \\
&\quad = \argmin_{(\vw, b) \in \mathcal{H}} \int_{\R^d} H\left[ \Phi(\vx^\tr \vw + b) \right] \mathcal{N}(\vx ; \vmu, \mSigma) d\vx \nonumber \\
&\quad = \argmin_{(\vw, b) \in \mathcal{H}} \int_{\R} H\left[ \Phi(t) \right] \mathcal{N}(t ; 0, \vw^\tr \mSigma \vw) dt \label{eq:step2} \\
&\quad = \argmax_{(\vw, b) \in \mathcal{H}} \vw^\tr \mSigma \vw \label{eq:step3}.
\end{align}
In (\ref{eq:step1}), we use (\ref{eq:utility}) and the fact that, as the hyperplane is passing through $\vmu$,
\begin{align*}
H \left[ \int_{\R^d} p(y = 1 | \vw, b, \vx) \mathcal{N}(\vx ; \vmu, \mSigma) d\vx \right] = H(1/2) = 1.
\end{align*}
In (\ref{eq:step2}), we use the fact that $\vx^\tr \vw + b \sim \mathcal{N}(\bm{0}, \vw^\tr \mSigma \vw)$, by properties of the Gaussian distribution.
Finally, in (\ref{eq:step3}), we start by noting that, for all $c_1, c_2$ such that $c_1 / c_2 > 1$, $H\left[ \Phi(c_1 t) \right] \le H\left[ \Phi(c_2 t) \right]$ for all $t$ with equality iff $t = 0$.
Hence, if $\tilde{\sigma}^2 > \sigma^2$, then
\begin{align*}
&\int_{\R} H\left[ \Phi(t) \right] \mathcal{N}(t ; 0, \tilde{\sigma}^2) dt
= \int_{\R} H\left[ \Phi(\tilde{\sigma} t) \right] \mathcal{N}(t ; 0, 1) dt \\
&< \int_{\R} H\left[ \Phi(\sigma t) \right] \mathcal{N}(t ; 0, 1) dt
= \int_{\R} H\left[ \Phi(t) \right] \mathcal{N}(t ; 0, \sigma^2) dt.
\end{align*}
Therefore, maximizing $\vw^\tr \mSigma \vw$ minimizes the expected entropy of $y$.
\end{proof}

\section{Proof of Theorem~\ref{thm:convergence}}

For simplicity, we will assume that $d = 1$; Section~\ref{sec:higherdim} explains how to generalize the result to any $d > 1$.
Denote by $x_t$ the location of the target object, and let $\mathcal{N} ( \hat{x} ; \mu_m , \sigma_m)$ be the belief about the target's location after $m$ observations.
Without loss of generality, let $\sigma_\varepsilon^2 = 1$.
In this case, the updates have the following form.
\begin{align}
\sigma^2_{m+1} &= \sigma^2_m + \beta_m \sigma^4_m, \label{eq:updsigma} \\
\mu_{m+1}      &= \mu_m      + \alpha_m \sigma^2_m \cdot z_m, \nonumber
\end{align}
where $z_m \in \{\pm 1\}$ with $\mathbf{P}(z_m = 1) = \Phi(x_t - \mu_m)$, and
\begin{align*}
\alpha_m &= c / \sqrt{\sigma_m^2 + 1}, \\
\beta_m  &= -c^2 / (\sigma_m^2 + 1), \\
c        &= \sqrt{2/\pi}.
\end{align*}

We start with a lemma that essentially states that $\sigma_m^2$ decreases as $1 / m$.
\begin{lemma}
\label{lem:sigdec}
For any initial $\sigma_0^2 > 0$ and for all $m \ge 0$, the posterior variance $\sigma_m^2$ can be bounded as
\begin{align*}
\frac{ \min\{0.1, \sigma_0^2\} }{m+1} \le \sigma_m^2 \le \frac{ \max\{10, \sigma_0^2\} }{m+1}.
\end{align*}
\end{lemma}
\begin{proof}
From (\ref{eq:updsigma}), we know that 
\begin{align*}
\sigma_{m+1}^2 = \left( 1 - c^2 \frac{\sigma_m^2}{\sigma_m^2 + 1} \right) \sigma_m^2
\end{align*}
First, we need to show that
\begin{align*}
f(x) = \left( 1 - c^2 \frac{x}{x + 1} \right) x
\end{align*}
is increasing on $\mathbf{R}_{>0}$.
This is easily verified by checking that
\begin{align*}
f'(x) = \left( 1 - c^2 \frac{x}{x + 1} \right) + x \left(1 - c^2 \frac{1}{(x + 1)^2} \right) \ge 0,
\end{align*}
for all $x \in \mathbf{R}_{>0}$
Next, we consider the upper bound.
Let $b = \max \{10, \sigma_0^2\}$.
We will show that $\sigma_m^2 \le b/ (m+1)$ by induction.
The basis step is immediate: by definition, $\sigma_0^2 \le b$.
The induction step is as follows, for $m \ge 1$.
\begin{align*}
&\sigma_{m}^2
    = \left( 1 - c^2 \frac{\sigma_{m-1}^2}{\sigma_{m-1}^2 + 1} \right) \sigma_{m-1}^2 \\
    &\qquad \le \left( 1 - c^2 \frac{b}{b + m} \right) \frac{b}{m}
    \le \frac{b}{m + 1} \\
&\iff 1 - c^2 \frac{b}{b + m} - \frac{m}{m + 1} \le 0 \\
&\iff b + m - c^2 (bm + b) \le 0 \\
&\iff m(1 - c^2 b) + b(1 - c^2) \le 1 - b(\underbrace{2c^2 - 1}_{\approx 0.27}) \le 0,
\end{align*}
where the first inequality holds because $f(x)$ is increasing.
The lower bound can be proved in a similar way.
\end{proof}

For completeness, we restate Theorem~\ref{thm:convergence} for $d = 1$.

\setcounter{theorem}{1}
\begin{theorem}[Case $d = 1$]
If the answers follow~\eqref{eq:probit}, then for any initial $\mu_0$ and $\sigma_0^2 > 0$ and as $m \to \infty$,
\begin{align*}
\sigma_m^2 &\to 0, \\
\mu_m      &\to x_t
\end{align*}
almost surely.
\end{theorem}

\begin{proof}
The first part of the theorem ($\sigma_m^2 \to 0$) is a trivial consequence of Lemma~\ref{lem:sigdec}.
The second part follows from the fact that our update procedure can be cast as the Robbins-Monro algorithm \cite{robbins1951stochastic} applied to $g(\mu) \doteq 2\Phi(x_t - \mu) - 1$, which has a unique root in $\mu = x_t$.
Indeed, $z_m$ is a stochastic estimate of the function $g$ at $\mu_m$, i.e., $\mathbf{E}(z_m) = 2\Phi(\mu_m - x_t) - 1$.
The remaining conditions to check are as follows.
\begin{itemize}
\item the learning rate $\gamma_m = \alpha_m \sigma_m^2$ satisfies
\begin{align*}
\sum_{m=0}^\infty \gamma_m   &\ge \frac{c \cdot \min\{\sigma_0^2, 0.1\}}{\sqrt{\sigma_0^2 + 1}} \sum_{m}^\infty \frac{1}{m} = \infty,\\
\sum_{m=0}^\infty \gamma_m^2 &\le ( c \cdot \max\{\sigma_0^2, 10.0\} )^2 \sum_{m}^\infty \frac{1}{m^2} < \infty
\end{align*}

\item $\lvert z_m \rvert \le 1$ for all $m$
\end{itemize}

Almost sure convergence then follows directly from the results derived in \cite{robbins1951stochastic} and \cite{blum1954approximation}.
\end{proof}

\subsection{Extending the proof to $d > 1$}
\label{sec:higherdim}

In order to understand how to extend the argument of the proof given above to $d > 1$, the following observation is key: Every query made during a search gives information on the position of the target along a \emph{single} dimension, i.e., the one perpendicular to the bisecting hyperplane.
This can be seen, e.g., from the update rule
\begin{align*}
\mSigma_{m+1} &= \left( \mSigma_m + \tau \vw_{i_m, j_m} \otimes \vw_{i_m, j_m}^T \right)^{-1},
\end{align*}
which reveals that the precision matrix (i.e., the inverse of the covariance matrix) is affected only in the subspace spanned by $\vw_{i_m, j_m}$.

Therefore, if we start with $\bm{\Sigma}_0 = \sigma_0^2 \bm{I}$, we can (without loss of generality) assume that the search procedure sequentially iterates over the dimensions.
At each iteration, the variance shrinks along that dimension only, leaving the other dimensions untouched.
Conceptually, we can think of the case $d > 1$ as interleaving $d$ independent one-dimensional search procedures.
Each of these one-dimensional searches converges to the corresponding coordinate of the target vector $\bm{x}_t$, and the variance along the corresponding dimension shrinks to $0$.

In general, one should consider the fact that the optimal hyperplane is not always unique, and the chosen hyperplane might not align with the current basis of the space.
This case can be taken care of by re-parametrizing the space by using a rotation matrix.
However, these technical details do not bring any new insight as to \emph{why} the result holds.

\section{Performance of Embedding Methods}

\begin{figure*}[htp]
\centering
\subfloat{\label{embedding_music}\includegraphics[width=8.5cm]{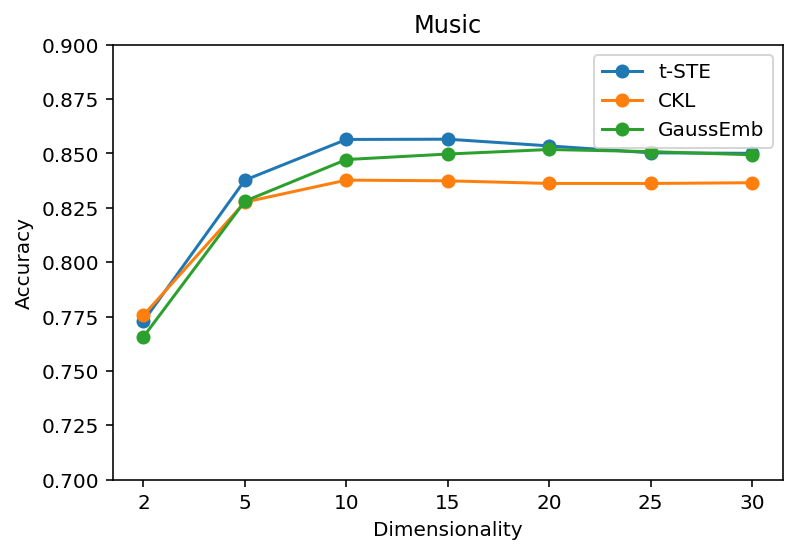}}
\subfloat{\label{embedding_food}\includegraphics[width=8.5cm]{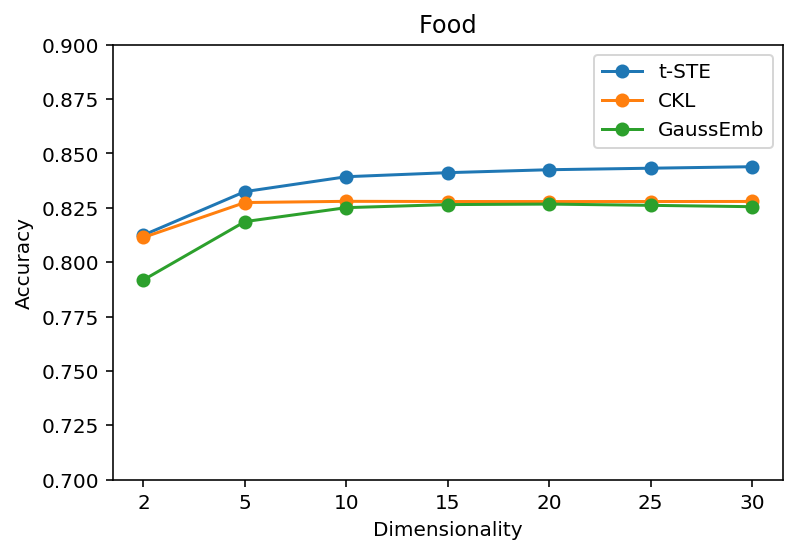}}
\caption{Evaluating embedding methods on two real world datasets: \emph{Music artists}, $n = 400$ music artists with $9090$ triplet comparisons after removing repeating and inconsistent triplets, and \emph{Food}, $n = 100$ images of food with $190376$ collected unique triplets.}
\label{embedding_plot}
\end{figure*}

\begin{figure*}[htp]
  \includegraphics[width=\textwidth]{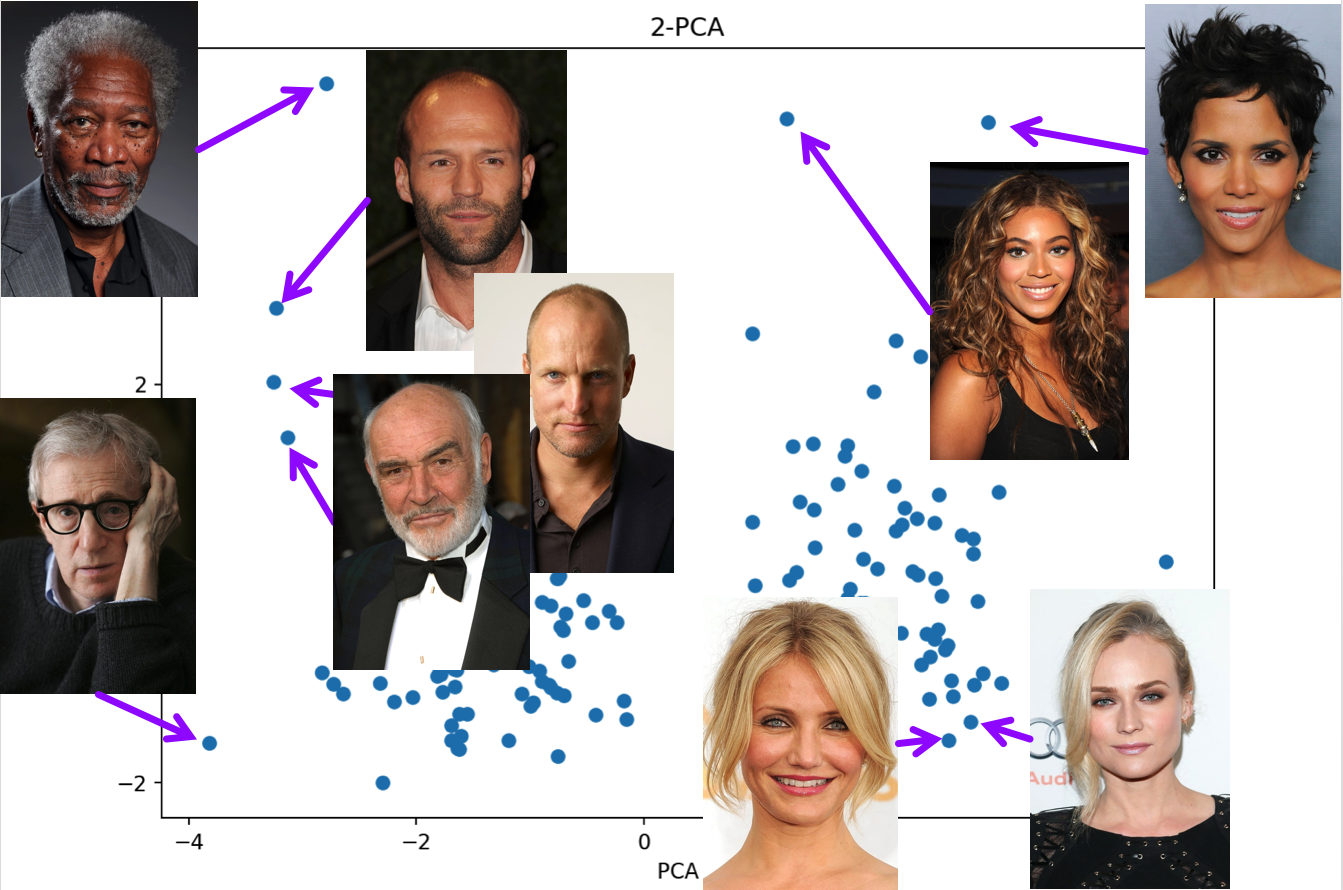}
  \caption{2-PCA on learned embedding.}
  \label{fig:pca_plot}
\end{figure*}

We evaluated the quality of the object embedding learned by our embedding technique \textsc{GaussEmb} on two real world datasets with crowdsourced triplet comparisons: \emph{Music artists} \cite{ellis2002quest} and \emph{Food} \cite{wilber2014cost}.

We compared our model to the state-of-the-art baselines, \textsc{CKL} and t-\textsc{STE}.
We measured accuracy---the percentage of satisfied triplets in the learned embeddings on a holdout set using 10-fold cross validation.
Since the ``true'' dimensionality of the feature space is not known a priori, we also vary the dimensionality $D$ of the estimated embedding between 2 and 30.
The results are presented in Fig.~\ref{embedding_plot}.

Overall, on both datasets, \textsc{GaussEmb} showed similar performance to t-\textsc{STE}, correctly modeling between 83\% and 85\% of triplets, and outperformed \textsc{CKL}.
We can conclude that the noise model considered in this paper reflects the real user behaviour on the comparison-like tasks well.

\section{Movie Actors Face Search Experiment}

We illustrate the results of performing 2-PCA on the learned embedding from the movie actors face search experiment in Fig.~\ref{fig:pca_plot}. 
It appears that the two principal components align well with gender and race. Note that this embedding was obtained solely from the triplet comparisons collected prior to the experiment (slightly more than 40000 triplets in total).

% \bibliographystyle{plain}
% \bibliography{learn2search}
% \bibliography{appendix}

% \end{document}

\end{document}